\documentclass[final,12pt]{colt2024} % Anonymized submission \documentclass[final,12pt]{colt2024} % Include author names
\pdfoutput=1
\title[Optimistic IDS]{Optimistic Information-Directed Sampling}
\usepackage{times}
\usepackage{skull}
\author[Neu, Papini, and Schwartz]{%
 \Name{Gergely Neu} \Email{gergely.neu@gmail.com}\\
 \addr Universitat Pompeu Fabra, Barcelona, Spain%
 \AND
 \Name{Matteo Papini} \Email{matteo.papini@polimi.it}\\
 \addr Politecnico di Milano, Milano, Italy%
 \AND
 \Name{Ludovic Schwartz} \Email{ludovic.v.schwartz76@gmail.com}\\
 \addr Universitat Pompeu Fabra, Barcelona, Spain%
}

\usepackage{enumitem}
\usepackage[normalem]{ulem}
\usepackage{preamble}
\newcommand{\Ber}{\mathrm{Ber}}
\newcommand{\tSIR}{surrogate information ratio\xspace}

\newcommand{\tADEC}{averaged Decision-to-Estimation-Coefficient\xspace}

\newcommand{\FGTS}{\textbf{\texttt{FGTS}}\xspace}
\newcommand{\IGW}{\textbf{\texttt{IGW}}\xspace}
\newcommand{\OIDS}{\textbf{\texttt{OIDS}}\xspace}
\newcommand{\OIDSSG}{\textbf{\texttt{OIDS-SG}}\xspace}
\newcommand{\IDS}{\textbf{\texttt{IDS}}\xspace}
\newcommand{\ETD}{\textbf{\texttt{E2D}}\xspace}
\newcommand{\VOIDS}{\textbf{\texttt{VOIDS}}\xspace}
\newcommand{\VOIDSG}{\textbf{\texttt{VOIDS-SG}}\xspace}
\newcommand{\ROIDS}{\textbf{\texttt{ROIDS}}\xspace}
\newcommand{\ROIDSG}{\textbf{\texttt{ROIDS-SG}}\xspace}
\newcommand{\SIG}{\overline{\IG}}
\newcommand{\SIGG}{\overline{\textup{IG}}^{\mathcal{G}}\xspace}%For gaussians
\newcommand{\SIR}{\overline{\IR}}
\newcommand{\SIRG}{\overline{\textup{IR}}^{\mathcal{G}}\xspace}

\newcommand{\ADEC}{\overline{\DEC}}
\newcommand{\ADECG}{\ADEC^{\mathcal{G}}}
\newcommand{\IG}{\textup{IG}\xspace}
\newcommand{\IGG}{\textup{IG}^{\mathcal{G}}\xspace}
\newcommand{\IR}{\textup{IR}\xspace}

\newcommand{\DEC}{\textup{DEC}\xspace}
\newcommand{\ODEC}{\textup{ODEC}\xspace}
\newcommand{\BDEC}{\textup{BDEC}\xspace}
\newcommand{\MAIR}{\textup{MAIR}\xspace}

\newcommand{\UE}{\textup{UE}}%Former Bellman Error
\newcommand{\OG}{\textup{OG}}%Former Feel Good Error
\newcommand{\lossmin}{\loss^*}
\newcommand{\blossmin}{\bloss^*}
\newcommand{\thetastar}{\theta_0}
\newcommand{\lt}{\ell_t}
\newcommand{\ltb}{\overline{\ell}_t}
\newcommand{\ptb}{\overline{p}_t}
\newcommand{\lts}{\lossmin_t}
\newcommand{\ltsb}{\blossmin_t}
\newcommand{\pIGW}{\pi^{\text{(IGW)}}}

\newcommand{\DDH}{\DD_H^2}
\newcommand{\DDKL}[2]{\DD_{\mathrm{KL}}\pa{#1\middle\|#2}}
\newcommand{\TV}{\mathrm{TV}}

\newcommand{\br}{\overline{r}}
\newcommand{\dd}{\mathrm{d}}
\newcommand{\bp}{\overline{p}}
 \newcommand{\X}{\mathcal{X}}
\newcommand{\F}{\mathcal{F}}
\newcommand{\A}{\mathcal{A}}

\newcommand{\real}{\mathbb{R}}

\newcommand{\DD}{\mathcal{D}}
\newcommand{\OO}{\mathcal{O}}

\newcommand{\DKL}[2]{\DD_{\text{KL}}\left(#1\middle\|#2\right)}
\newcommand{\II}[1]{\mathbb{I}_{\left\{#1\right\}}}
\newcommand{\PP}[1]{\mathbb{P}\left[#1\right]}

\newcommand{\EE}[1]{\mathbb{E}\left[#1\right]}

\newcommand{\EEs}[2]{\mathbb{E}_{#2}\left[#1\right]}

\newcommand{\PPt}[1]{\mathbb{P}_t\left[#1\right]}
\newcommand{\EEt}[1]{\mathbb{E}_t\left[#1\right]}

\newcommand{\PPcc}[2]{\mathbb{P}\left[\left.#1\right|#2\right]}

\newcommand{\EEcc}[2]{\mathbb{E}\left[\left.#1\right|#2\right]}

\newcommand{\EEcct}[2]{\mathbb{E}_t\left[\left.#1\right|#2\right]}

\def\argmin{\mathop{\mbox{ arg\,min}}}
\def\argmax{\mathop{\mbox{ arg\,max}}}
\newcommand{\ra}{\rightarrow}

\newcommand{\siprod}[2]{\langle#1,#2\rangle}
\newcommand{\iprod}[2]{\left\langle#1,#2\right\rangle}

\newcommand{\norm}[1]{\left\|#1\right\|}

\newcommand{\onenorm}[1]{\norm{#1}_1}

\newcommand{\ev}[1]{\left\{#1\right\}}
\newcommand{\pa}[1]{\left(#1\right)}
\newcommand{\bpa}[1]{\bigl(#1\bigr)}

\newcommand{\wh}{\widehat}

\newcommand{\loss}{\ell}

\newcommand{\hp}{\wh{p}}

\newcommand{\bloss}{\bar{\loss}}

\newcommand{\htheta}{\wh{\theta}}

\definecolor{PalePurp}{rgb}{0.66,0.57,0.66}

\newcommand{\R}{\boldsymbol{R}}

\newcommand{\hL}{\wh{L}}

\newtheorem{thm}{Theorem}
\newtheorem{cor}{Corollary}
\begin{document}

\maketitle

\begin{abstract}%
	We study the problem of online learning in contextual bandit problems where the loss function is assumed to belong to a 
	known parametric function class. We propose a new analytic framework for this setting that bridges the Bayesian theory 
	of information-directed sampling due to \citet{Russo_V17} and the worst-case theory of \citet*{Foste_K_Q_R22} based 
	on the decision-estimation coefficient. Drawing from both lines of work, we propose a algorithmic template called 
	Optimistic Information-Directed Sampling and show that it can achieve instance-dependent regret guarantees similar to 
	the ones achievable by the classic Bayesian IDS method, but with the major advantage of not requiring any Bayesian 
	assumptions. The key technical innovation of our analysis is introducing an optimistic surrogate model for the regret 
	and using it to define a frequentist version of the Information Ratio of \citet{Russo_V17}, and a less conservative 
	version of the Decision Estimation Coefficient of \citet{Foste_K_Q_R22}.
\end{abstract}

\begin{keywords}%
	Contextual bandits, information-directed sampling, decision estimation coefficient, first-order regret bounds.%
\end{keywords}

\section{Introduction}
We present a framework for the analysis of a family of sequential decision-making algorithms known as 
Information-Directed Sampling (\IDS). First proposed by \citet{Russo_V17}, \IDS is a Bayesian algorithm that selects 
its policies by optimizing a measure called the \emph{information-ratio}, which measures the tradeoff between 
instantaneous regret and information gain about the problem instance at hand. In a Bayesian setup, both components of 
the information ratio are functions of the posterior distribution over models, and can thus be explicitly 
calculated. As shown by \citet{Russo_V17}, the resulting algorithm can guarantee massive statistical gains over more 
common approaches like Thompson sampling \citep{Tho33} or optimistic exploration methods \citep{LR85}, and in 
particular can take advantage of the structure of the problem instance much more effectively. Realizing the same gains 
in a non-Bayesian setup (which we will sometimes call \emph{frequentist}, for lack of a better word) is hard for 
multiple reasons, the most severe obstacle being that the true model is entirely unknown and Bayesian posteriors cannot 
be used to quantify the uncertainty about 
the model in a meaningful way. As such, defining appropriate notions of information gain and information ratio is not 
straightforward. This is the problem we address in this paper. \looseness=-1

Our main contribution is constructing a version of information-directed sampling that is implementable without Bayesian 
assumptions, and yields frequentist versions of the same problem-dependent guarantees as the ones achieved by the 
original \IDS method in a Bayesian setup. The key element in our approach is the introduction of a \emph{surrogate 
model} that allows for a meaningful definition of the information ratio that is amenable to a frequentist analysis. 
This surrogate model is the function of an optimistically adjusted posterior distribution inspired by the ``feel-good 
Thompson sampling'' algorithm of \citet{Zhang_21}, and is used to estimate the components of the information ratio: the 
regret and the information gain. With these components, it becomes possible to define an information ratio that is an 
explicit function of the optimistic posterior, which can then be optimized to yield a decision-making rule that we call 
``optimistic information-directed sampling'' (\OIDS).

For the sake of concreteness, we focus on the problem of contextual bandits and show that \OIDS can not only recover 
worst-case optimal regret bounds in this case, but also satisfies problem-dependent guarantees that are commonly 
referred to as \emph{first-order bounds}~\citep{cesa2005improved,Agar_K_L_L17,Allen_B_L18,Foste_K21}. Besides these general guarantees, we also provide some 
illustrative examples that show that \OIDS can reproduce the expedited learning behavior of \IDS on easy problems, but 
without requiring Bayesian assumptions.

Our methodology also draws inspiration from the analytic framework of \citet*{Foste_K_Q_R22}, developed for a very 
general range of sequential decision-making problems. Their analysis revolves around the notion of the 
\emph{decision-estimation coefficient} (\DEC), which quantifies the tradeoffs that need to be made between achieving 
low regret and gaining information about the true model in a way that is similar to the information ratio of 
\citet{Russo_V15}. The main contribution of \citet{Foste_K_Q_R22} is showing that the minimax regret in any 
sequential decision-making problem can be lower bounded in terms of the DEC, and they also show that nearly 
matching upper bounds can be achieved via a simple algorithm they call \emph{estimation to decisions} (\ETD). 
Unlike the information ratio, the DEC does not make use of a Bayesian posterior to quantify uncertainty, but is rather 
defined as a worst-case notion, and as such provides frequentist guarantees that hold uniformly for all problem 
instances. However, the worst-case nature of the DEC can also be seen as an inherent limitation of their framework. In 
particular, the \ETD algorithm is also based on the same conservative notion of regret-information tradeoff, and thus 
all known guarantees for this algorithm (and its variants such as the ones proposed by 
\citealp{Chen_M_B22,Foste_G_H23,Foste_G_Q_R_S23,Kirsch_B_C_T_Sz23}) fail to take advantage of problem structures that 
may facilitate fast learning.\looseness=-1

Our own framework unifies the advantages of the two threads of literature described above: unlike \ETD, it is able to 
achieve instance-dependent guarantees and learn faster in problems with more structure, and, unlike standard \IDS, it can 
do so without relying on Bayesian assumptions. Our analysis draws on elements of both lines of work, and also on the 
techniques introduced by \citet{Zhang_21}, as mentioned above.

We are not the first to attempt the generalization of \IDS beyond the Bayesian setting. 
\citet{Kirsc_K18} proposed a frequentist alternative to the information ratio for the special case of loss functions 
that are linear in some unknown parameter, and constructed an appropriate version of \IDS that is able to take 
advantage of certain problem structures and obtain guarantees that improve upon the minimax rates. Their approach has 
inspired a line of work aiming to prove tighter and tighter problem-dependent bounds for a range of sequential 
decision-making problems, but so far all of these results remained limited to linearly structured losses and 
observations \citep{Kirsc_L_K20,Kirsc_L_V_Sz21,Hao_L_Q22}. In contrast, our notion of information ratio does not 
require any specific problem structure like linearity. Another notable work is that of \citet{XZ23}, who propose and 
analyze a version of \IDS based on optimizing a version of the DEC called the ``algorithmic information ratio'', and is 
subject to the same limitations as the DEC variants mentioned above. We discuss this work in more detail in 
Appendix~\ref{app:complexities}.

\paragraph{Notation.} The squared Hellinger distance between two probability distributions $P$ and $P'$ (with a 
common dominating measure $Q$) is defined as $\DDH\pa{P,P'} = \frac 12 \int \bpa{\sqrt{\frac{\dd P}{\dd Q}} - 
\sqrt{\frac{\dd P'}{\dd Q}}}^2 \dd Q$, and the relative entropy (or Kullback--Leibler divergence) as $\DKL{P}{P'} = \int 
\log \frac{\dd P}{\dd P'} \dd P$.

\section{Preliminaries}
We study contextual bandit problems with finite action spaces and parametric loss functions.
The sequential interaction 
scheme between the \emph{learner} and the \emph{environment} consists of the following steps being repeated for a 
sequence of rounds $t=1,2,\dots,T$:
\begin{itemize}[itemsep=0pt,parsep=3pt,partopsep=6pt,topsep=6pt]
	\item The environment picks a context $ X_t \in \X $, possibly using randomization and taking into account the 
history of actions, losses and contexts,
	\item the learner observes $X_t$ and picks an action $ A_t \in \mathcal{A} $, possibly using randomization and 
taking into account the history of actions, losses and contexts,
	\item the learner incurs a loss $ L_t$, drawn independently of the past from a fixed distribution that depends on 
$X_t,A_t$.
\end{itemize}
We denote the sigma-algebra generated by the interaction history between the learner and the environment up to 
the end of round $t$ as $\F_{t} = \sigma(X_1,A_1,L_1,\dots,X_{t},A_{t},L_{t})$, and the
probabilities and expectations conditioned on the history as $\PPt{\cdot} = \PPcc{\cdot}{\F_{t-1},X_t}$ and 
$\EEt{\cdot} = \EEcc{\cdot}{\F_{t-1},X_t}$. 

We will suppose that the action space is finite with cardinality $|\mathcal{A}| = K$, and that the loss function 
belongs to a known parametric class, but is otherwise unknown to the learner. Specifically, we assume that there is a 
known parameter space $ \Theta $ that parametrizes a class of loss functions $ \ell : \Theta\times \mathcal{X} \times 
\mathcal{A} \rightarrow \R $, and a true parameter $\thetastar \in \Theta$ such that $\EEcct{L_t}{X_t,A_t} = 
\ell(\thetastar,X_t,A_t)$. We will refer to this condition as \emph{realizability}. The distribution of random 
losses 
under parameter $\theta$ generated in response to taking action $a$ in context $x$ will be denoted by $p(\theta,x,a)$,
and we will write $ p(\cdot|\theta,x,a) $ to designate the corresponding density with respect to a reference 
measure (usually the counting measure or Lebesgue measure). Unless stated otherwise, we will assume that the loss 
distribution is fully supported on the interval $[0,1]$ for all parameters $\theta$. Furthermore, we will often 
abbreviate $\ell(\theta,X_t,a)$ as $\ell_t(\theta,a)$ and $p(\theta,X_t,a)$ as $p_t(\theta,a)$ to lighten our notation.
Our formulation will make central use of \emph{policies} which prescribe randomized behavior rules for the 
learning agent. Precisely, a policy $\pi:\X\ra \Delta_{\A}$ maps each context $x$ to a distribution over actions 
denoted as $\pi(\cdot|x)$. Since we will mostly work with action distributions conditioned on the fixed contexts $X_t$, 
we will mostly represent policies as distributions over actions, and use the same notation $\pi\in\Delta_{\A}$ for this 
purpose. We will focus on learning algorithms that, in each round $t$, select a randomized policy $\pi_t \in 
\Delta_{\A}$ based on the interaction history $\F_{t-1}$ and $X_t$. We also define the \emph{optimal loss} in round $t$ 
under model parameter $\theta$ as $\lossmin_t(\theta) = \min_a \ell_t(\theta,a)$. 
The agent aims to make its decisions in a way in that minimizes the expected sum of losses, and in 
particular aims to incur nearly as little loss as the true optimal policy. The extent to which the learner succeeds in 
achieving this goal is measured by the (total expected) \emph{regret} defined as
\begin{equation}
\label{eq:Regret_definition}
R_T(\thetastar) = \EE{\sum_{t=1}^{T}\left( \ell_t(\thetastar,A_t)-\lts(\thetastar) 
\right)}.
\end{equation}
The expectation is over all sources of randomness: the agent's randomization over actions, the adversary's 
randomization over contexts and the randomness of the realization of the losses. We also define \emph{instantaneous 
regret} of an action $a$ under parameter $\theta$ for each $t$ as
\[
 r_t(a;\theta) = \ell_t(\theta,a)-\lossmin_t(\theta),
\]
and the instantaneous regret of policy $\pi$ as $r_t(\pi;\theta) = \sum_a \pi(a) r_t(a;\theta)$. With this 
notation, the regret of the online learning algorithm can be written as $R_T(\thetastar) = 
\mathbb{E}\bigl[\sum_{t=1}^T r_t(\pi_t;\thetastar)\bigr]$.

\section{Two competing theories of sequential decision making}\label{sec:frameworks}
Our work connects two well-established analytic frameworks for sequential decision making: the Bayesian framework of 
\citet{Russo_V17} and the worst-case framework of \citet*{Foste_K_Q_R22}. We review the two in some detail below, 
highlighting some of their merits and limitations that we address in this paper, and relegate a more detailed 
discussion of variants of these frameworks to Appendix~\ref{app:complexities}.

\subsection{The information ratio and Bayesian information-directed sampling}\label{sec:IDS}
The influential work of \citet{Russo_V15,Russo_V17} set forth an analytic framework based on a Bayesian 
learning paradigm where the true model parameter $\thetastar$ is supposed to be sampled from a known prior distribution 
$Q_0 \in \Delta_{\Theta}$, and the performance of the learner is measured on expectation with respect to this random 
choice of instance. We refer to the expected regret under this prior as the \emph{Bayesian regret}. 
Their work has established that the Bayesian regret of any algorithm can be 
upper bounded in terms of a quantity called the \textit{Information Ratio} (IR). 
For the sake of exposition, we will follow the setup and notation of \citet{Neu_O_P_S22}, who 
study the Bayesian version of our contextual bandit setting, and define the information ratio of policy $\pi$ in the 
$t$-th round of interaction as\looseness=-1
\begin{equation}
\label{eq:LIR}
\rho_t(\pi) = \frac{\bpa{\EEs{r_t(\pi;\thetastar)}{\thetastar\sim Q_t}}^2}{\IG_t(\pi)}.
\end{equation}
In the above expression, both the numerator and the denominator are functions of the \emph{posterior distribution} 
$Q_t$ of the parameter $\thetastar$, computed based on all information available to the learner up to the beginning of 
round $t$. Specifically, the numerator is the squared expected regret in round $t$, where the expectation is taken 
under the posterior distribution $Q_t$, and the denominator is an appropriately defined measure of \emph{information 
gain} that serves to quantify the amount of new information revealed about $\thetastar$ after having observed the 
latest 
loss $L_t$. The information gain is formally defined as 
\begin{equation}\label{eq:IG}
 \IG_t(\pi) = \sum_a \pi(a) \int \DKL{p_t(\theta,a)}{\bp_t(a)} \dd Q_t(\theta),
\end{equation}
where $\bp_t(a) = \int p_t(\theta,a) \dd Q_t(\theta)$ is the posterior predictive distribution of the loss 
$L_t$ given that action $a$ is played in context $X_t$. In other words, the information gain is
the \emph{mutual information} between the posterior-sample parameter $\theta_t \sim Q_t$ and a randomly sampled loss 
$\hL_t \sim p_t(\theta_t,a)$. 

Given the above definitions, \citet{Russo_V15,Russo_V17} show that the Bayesian regret of \emph{any} algorithm can be 
upper bounded as follows:
\begin{equation}
\begin{split}
\label{eq:BR_bound_LIR}
	\EEs{R_T(\thetastar)}{\thetastar\sim Q_0}
	\le \sqrt{\EE{\sum_{t=1}^{T}\rho_t(\pi_t)}\cdot \EE{\sum_{t=1}^{T}\IG(\pi_t)}}.
\end{split}
\end{equation}
The second sum above can be upper bounded by the entropy of $\thetastar$ under the prior distribution, regardless of 
what 
algorithm is used to select the sequence of policies. This suggests that one can achieve low regret by picking the 
sequence of policies in a way that minimizes the information ratio: $\pi_t = \argmin_\pi \rho_t(\pi)$. This algorithm 
is called \emph{information-directed sampling} (\IDS), and has been shown to achieve regret guarantees that often 
improve significantly over worst-case bounds achieved by more traditional methods based on posterior sampling or 
optimistic exploration methods. In particular, for the contextual bandit setting we study in this paper, the works 
of \citet{Neu_O_P_S22} and \citet{Min_R23} have shown that the information ratio of \IDS is bounded by the number of 
actions $K$. When the parameter space is finite with cardinality $N$, this result implies that the algorithm 
achieves the minimax optimal regret bound of $ \OO(\sqrt{KT\log N})$ for this Bayesian setting. 

Despite their appealing properties, \IDS-style methods have however remained largely limited to the Bayesian setting, 
as there appears to be no universal way of defining an algorithmically useful information ratio without Bayesian 
assumptions. In particular, the instantaneous regret $r_t(\pi;\thetastar)$ cannot be computed without knowledge of 
$\thetastar$, and there is no reason to believe that the information gain defined in terms of a Bayesian posterior 
would 
meaningfully measure the reduction in uncertainty about $\thetastar$ in this more general setting.

\subsection{The decision-estimation coefficient and the estimations-to-decisions algorithm}\label{sec:DEC}
The fundamental work of \citet{Foste_K_Q_R22} provides a general theory of sequential decision making, providing a 
range of upper and lower bounds depending on a quantity they call the \emph{decision-estimation coefficient} (DEC). 
With a little deviation from their notation and terminology, the DEC associated with a policy $\pi$, a model class 
$\Theta$ and a ``reference model'' $\hp_t:\A\ra \Delta_{\real}$ 
is defined as
\begin{equation}\label{eq:DEC}
  \DEC_{\gamma,t}(\pi;\Theta,\hp) = \sup_{\theta\in\Theta} \sum_{a} \pi(a) \pa{\ell(\theta,X_t,a) - 
\ell(\theta,X_t,\pi_\theta) - \gamma \DDH\pa{p_t(\theta,a), \hp_t(a)}},
\end{equation}
where $\gamma > 0$ is a trade-off parameter. With this notation, \citet{Foste_K_Q_R22} define the decision-estimation 
coefficient associated with the model class $\Theta$ as
\[
  \DEC_{\gamma}(\Theta) = \sup_t \sup_{\hp \in \Delta_{\real}} \inf_{\pi \in \Delta_{\A}} 
\DEC_{\gamma,t}(\pi;\Theta,\hp).
\]
Besides the remarkable feat of showing that the minimax regret can be lower bounded in terms of the above quantity, 
they also show that nearly 
matching upper bounds can be achieved via a simple algorithm they call \emph{estimation to decisions} (\ETD). In each 
round $t$, \ETD takes as input a reference model $\hp_t$ and outputs the policy achieving the minimum in the definition 
of the DEC: $\pi_t = \argmin_{\pi} \DEC_{\gamma,t}(\pi;\Theta,\hp_t)$.
They show that the regret of this method can be upper bounded in terms of the DEC as follows:
\begin{align*}
 R_T(\thetastar) 
&\le \DEC_{\gamma}(\Theta)\cdot T + \gamma  \sum_{t=1}^T \DDH\pa{p_t(\thetastar,a), \hp_t(a)}.
\end{align*}
This shows that the regret of \ETD can be upper bounded as the sum of the DEC of the model class $\Theta$ and the total 
\emph{estimation error} associated with the sequence of predictions $\hp_t$ (measured in terms of Hellinger distance). 
For the contextual bandit setting with finite parameter class of size $N$, they show that the 
total estimation error can be upper bounded by $\gamma \log N$ (under an appropriate choice of the predictions 
$\hp_t$),  and that the DEC is upper bounded by $K/\gamma$, which once again recovers the minimax 
optimal rate of order $ \OO(\sqrt{KT\log N})$ when $\gamma$ is tuned correctly.

A significant problem with the approach outlined above is that the DEC is an inherently worst-case measure of 
complexity due to the supremum taken over $\theta$ in its definition~\eqref{eq:DEC}. 
Since the \ETD algorithm itself is based on this possibly loose bound on the regret-to-information gap, this looseness 
may not only affect the bound but also the actual performance of the algorithm. Intuitively, one may hope to be able to 
do better by replacing the supremum over model parameters by only considering models that are still ``statistically 
plausible'' in an appropriate sense. In what follows, we provide an algorithm that realizes this potential.

\section{Optimistic information-directed sampling}\label{sec:OIDS}
Our approach solves the issues outlined in the previous sections with both the Bayesian information ratio and the 
decision estimation coefficient. In particular, our method will extend Bayesian \IDS by being able to provide 
non-Bayesian performance guarantees, and will be able to address the over-conservative nature of the DEC and provide 
strong instance-dependent guarantees. 

Following \citet{Zhang_21}, we start by defining the \emph{optimistic posterior} $Q_t^+\in \Delta_{\Theta}$ via the 
following recursive update rule (starting from an arbitrary prior $Q_1^+(\theta)\in\Delta_{\Theta}$):
\begin{equation}
\label{eq:Optimistic_posterior}
	\frac{dQ_{t+1}^+}{dQ_t^+}(\theta) \propto (p_t(L_t|\theta,A_t))^\eta\cdot \exp(-\lambda\cdot 
\lossmin_t(\theta)).
\end{equation}
Here, $\eta$ and $\lambda$ are positive constants that will be specified later. For now, we will only say that $\eta$ 
should be thought of as a ``large'' constant of order $1$, and $\lambda$ as a ``small'' parameter of order $1/\sqrt{T}$ 
in the worst case. To proceed, we define the \emph{optimistic posterior predictive distribution} of the loss for each 
$t$ and $a$ as the mixture $\bp_t(a) = \int p_t(\theta, a)dQ_t^{+}(\theta)$, and the \emph{surrogate loss function} and 
\emph{surrogate optimal loss function} respectively as 
\begin{equation}
 \bloss_t(a) =\int \ell_t(\theta,a)\dd Q_t^{+}(\theta) \qquad\mbox{and}\qquad \blossmin_t =\int 
\lossmin_t(\theta)\dd Q_t^{+}(\theta).
\end{equation}
In words, these quantities are averages with respect to a mixture model over all contextual bandit instances with 
mixture weights given by the optimistic posterior $Q_t^+$. Notably, they are \emph{improper} estimators of the true 
likelihood, loss, and optimal loss functions respectively, as there may be no single $\theta\in\Theta$ that corresponds 
to these exact functions (unless one assumes certain convexity properties of the relevant objects). With these 
notations, we define the \emph{surrogate regret} of policy $\pi$ in round $t$ as 
$\br_t(\pi) = \bloss_t(\pi) - \ltsb$. As we will see in the analysis, the optimistic posterior plays 
a key role in ensuring that the surrogate regret does not overestimate the true regret by too much on average, which 
makes it a sensible target for minimization.

It remains to define our notion of information gain that we will call \emph{surrogate information gain}. 
Formally, this quantity is defined for each policy $ \pi $ as follows:
\begin{equation}
\label{eq:SIG}
 	\SIG_{t}(\pi) = \sum_{a\in \A} \pi(a)\int 
 \DDH\bigl(p_t(\theta,a), \bp_{t}(a)\bigr)\dd Q_t^{+}(\theta).
\end{equation}
Notably, this definition matches the original notion of information gain used by \citet{Russo_V15,Russo_V17}, up to the 
differences that the divergence being used is the squared Hellinger divergence instead of Shannon's relative entropy, 
and that the expectation is taken over the optimistic posterior instead of the plain Bayesian posterior.
We will sometimes write  $\br_t(\pi;Q_t^+)$ and $\SIG_{t}(\pi;Q_t^+)$ to emphasize that these are functions of the 
optimistic posterior $Q_t^+$.
With the above definitions, we are now ready to introduce the central quantity of our algorithmic framework and our 
analysis: the \emph{\tSIR} defined for each policy $\pi$ as
\begin{equation}
\label{eq:SIR}
	\SIR_t(\pi) = \frac{\pa{\br_t(\pi)}^2}{\SIG_t(\pi)} = \frac{\left( \sum_{a\in\A} \pi(a)
\int \bpa{\loss_t(\theta,a)-\lossmin_t(\theta) }\dd Q_t^+(\theta) \right)^{2}}{\sum_{a\in \A} \pi(a)\int 
\DDH\bpa{\bp_{t}(a),p_t(\theta,a)}\,\dd Q_t^{+}(\theta)}.
\end{equation}
Importantly, computing the \tSIR does not require knowledge of  $\thetastar$: both its denominator and numerator can be 
expressed in terms of the optimistic posterior $Q_t^+$.  To emphasize this 
fact, we will sometimes write $\SIR_t(\pi;Q_t^+)$ for $\SIR_t(\pi) $.

We will also define the ``offset'' counterpart of the surrogate information ratio that is more closely related to the 
decision-estimation coefficient of \citet{Foste_K_Q_R22}. Following the terminology introduced in 
Section~\ref{sec:DEC}, we introduce the \emph{averaged 
decision-estimation coefficient} (ADEC) of policy $\pi$ for each $\mu > 0$ as
\begin{equation}
\label{eq:VSIR}
\begin{split}
\ADEC_{\mu,t}(\pi) &= \bar{r}_t(\pi) - \mu\cdot \SIG_t(\pi) 
\\ 
&= \sum_a \pi(a) \int \pa{\ell_t(\theta,a) - 
\lossmin_t(\theta) - \mu \DDH\bpa{\bp_{t}(a),p_t(\theta,a)}} \dd Q_t^+(\theta)
% &= \sum_a \pi(a) \int \pa{\ell_t(\theta,\pi) - 
% \lossmin_t(\theta) - \mu \DDH\pa{\ell_t(\theta,\pi), \bloss_t(\pi)}} \dd Q_t^+(\theta).
\end{split}
\end{equation}
Once again, we also define the notation $\ADEC_{\mu,t}(\pi;Q_t^+) = \ADEC_{\mu,t}(\pi)$ to emphasize the 
dependence of the ADEC on the posterior distribution $Q_t^+$.
This definition departs from the classic DEC in that, instead of taking a supremum over model parameters, it 
is defined via an expectation with respect to the optimistic posterior, thus preventing overly conservative choices of 
$\theta$. It should be clear from this definition that the ADEC is always smaller than its original counterpart defined 
by \citet{Foste_K_Q_R22}, as long the latter uses the optimistic posterior predictive distribution as its reference 
model: $\ADEC_{\mu,t}(\pi;Q_t^+) \le \DEC_{\mu,t}(\pi;\bp_t,\Theta)$. 

The surrogate information ratio and the ADEC are related to each other by the inequality 
\begin{equation}
\label{eq:IR_VIR}
	\ADEC_{\mu,t}(\pi) \leq \frac{\SIR_t(\pi)}{4\mu}
\end{equation}
that holds for all $\mu > 0$. Conversely, it can be seen that 
\begin{equation}
\label{eq:VIR_IR}
	\SIR_t(\pi) = \inf\ev{C > 0:\, \ADEC_{\mu,t}(\pi) \leq \frac{C}{4\mu} \,\,\,\,\pa{\forall \mu > 0}}.
\end{equation}
These are both direct consequences of the inequality of arithmetic and geometric means. That is, whenever the ADEC
behaves as $C_t/\mu$ for all $\mu$, the surrogate information ratio succinctly summarizes its behavior at all levels 
$\mu$. We will dedicate special attention to this case below, but we also note that there are several important cases 
where the ADEC behaves differently, and the information ratio is a less appropriate notion of complexity. We defer 
further discussion of this to Section~\ref{sec:discussion}.

With the above notions, we are now ready to define the algorithmic framework we study in this paper, with two separate 
versions depending on whether we consider the surrogate information ratio or the average DEC as the basis of decision 
making. Both versions are referred to as \emph{optimistic information-directed sampling} (optimistic \IDS or \OIDS). 
Following 
the terminology of \citet{Hao_L22}, we call the first variant which selects its policies as $\pi_t = \argmin_{\pi} 
\SIR(\pi;Q_t^+)$ \emph{vanilla optimistic information-directed sampling} (\VOIDS), and the second variant that selects 
$\pi_t = \argmin_{\pi} \ADEC_\mu(\pi;Q_t^+)$ \emph{regularized optimistic information-directed sampling} (\ROIDS). We 
provide the pseudocode for these methods for quick reference as Algorithm~\ref{alg:OIDS}.

\begin{algorithm}
	\caption{Optimistic Information Directed Sampling (\OIDS)}
	\label{alg:OIDS}
		\textbf{Input:} prior $Q_1^{+}$, parameters $ \eta$, $\lambda$, $\mu $.
		\\
		\textbf{For} $ t=1,\ldots,T $, \textbf{repeat}:
		\begin{enumerate}[itemsep=0pt,parsep=3pt,partopsep=3pt,topsep=3pt]
		\item Observe context $X_t$,
		\item[2a.] \VOIDS: play policy $\pi_t=\argmin_{\pi\in \Delta(\mathcal{A})} \SIR_t(\pi,Q_t^{+})$,
		\item[2b.] \ROIDS: play policy $\pi_t=\argmin_{\pi\in \Delta(\mathcal{A})} \ADEC_t(\pi,Q_t^{+}, 
\mu)$,
		\item[3.] incur loss $L_t$,
		\item[4.] update optimistic prior, $ Q_{t+1}^{+}(\cdot ) \propto Q_t^{+}(\cdot ) (p_t(\cdot,A_t,L_t))^{\eta}\exp{(-\lambda\lts(\cdot))}$.
		\end{enumerate}
\end{algorithm}

\section{Main results}
We now present our main results regarding the two varieties of our optimistic \IDS algorithm. We first show a general 
worst-case regret bound stated in terms of the time horizon $T$ and the information ratio. More importantly, we also 
show instance-dependent guarantees on the performance of \OIDS that replace the scaling with $T$ in the upper bounds by 
the total loss of the best policy after $T$ steps. For simplicity of exposition and easy comparison with existing 
results, we will present our main results assuming that the parameter space $\Theta$ is finite with cardinality $N$, 
and that the losses are almost surely bounded in the interval $[0,1]$. 
We extend these results to compact metric parameter spaces in Section~\ref{subsec:metric_parameter_space}, and provide 
an extension to subgaussian losses in Section~\ref{subsec:subgaussian_losses}. Besides these general results, we also 
present several examples where \OIDS can achieve very low regret by exploiting various flavors of problem structure, in 
Appendix~\ref{app:examples}.

\subsection{Worst-case bounds}\label{sec:worstcase}
We start by stating a general worst-case regret bound that relates the regret of any 
algorithm to its \tSIR. This result is the non-Bayesian counterpart of the bounds stated in 
\citet{Russo_V17}, \citet{Hao_L22} and \citet{Neu_O_P_S22} in that it basically says that any algorithm with bounded 
information ratio will enjoy bounded regret. 
\begin{thm}
\label{thm:regret_IR}
Assume $ |\Theta| = N < \infty $ and let $\lambda > 0$ be arbitrary. Then, for any choice of prior 
$Q_1\in\Delta_{\Theta}$, the regret of any algorithm satisfies the following bound:
\begin{equation}
\begin{split}
\label{eq:regret_IR}
\EE{R_T(\theta_0)} &\leq \frac{\log \frac{1}{Q_1(\thetastar)}}{\lambda} + \lambda T \cdot \left( 
\frac{\sum_{t=1}^{T}\EE{\ADEC_{1/10\lambda,t}(\pi_t;Q_t^+)}}{\lambda T} + \frac{21}{4} \right)\\
	&\leq \frac{\log \frac{1}{Q_1(\thetastar)}}{\lambda} + \lambda T \cdot \left( 10\cdot 
\frac{\sum_{t=1}^{T}\EE{\SIR_t(\pi_t;Q_t^+)}}{T} + 
\frac{21}{4} \right).
\end{split}
\end{equation}
\end{thm}
We provide a proof sketch, with pointers to the full technical proof details, in Section~\ref{sec:worst-case-analysis}.
As is common in the information directed sampling literature, we will turn this guarantee into a more concrete bound 
on the regret of \OIDS by exhibiting a ``forerunner" algorithm that is able to control the \tSIR and is relatively easier to analyze. Indeed, this will certify a regret 
bound for \OIDS, since the latter precisely minimizes the \tSIR at every round, and as such is guaranteed to 
achieve the same or a better bound. In particular, we use the \emph{feel-good Thompson sampling} 
(\FGTS) algorithm of \citet{Zhang_21} as our forerunner, which samples a parameter $\theta_t$ from the optimistic 
posterior and then plays the policy $\pi_t = \argmax_{\pi} \sum_{a} \pi(a) \ell_t(\theta_t,a)$. 
\begin{lemma}
\label{lemma:IR_TS}
The surrogate information ratio and averaged decision-to-estimation-coefficient of \VOIDS and \ROIDS satisfy for any $ \mu \geq 0 $
\begin{equation}
	4 \mu \ADEC_{\mu,t}(\ROIDS) \leq 4\mu \ADEC_{\mu,t}(\VOIDS) \leq \SIR_t(\VOIDS) \leq \SIR_t(\FGTS) \leq 8K.  
\end{equation}

\end{lemma}
We note that the above result is more of a property of the posterior sampling policy than \FGTS itself, as the bound 
holds for any distribution that is handed to \OIDS. This result is not especially new: similar statements have been 
proven in a variety of papers including \citet{Russo_V15,Russo_V17,Zhang_21,Foste_K_Q_R22,Neu_O_P_S22}. We provide a 
proof in Appendix~\ref{subsubsec:IR_TS}.
Putting the two previous results together, we get the following upper bound on the regret of \OIDS:
\begin{cor}
\label{cor:regret_IDS}
Assume $ |\Theta| = N < \infty  $, and let $ \lambda = \sqrt{\frac{\log N}{(80K + \frac{21}{4})T}} $. Then, the regret of \ROIDS 
with input parameter $\mu = \frac{1}{10 \lambda}$ and \VOIDS both satisfy
\begin{equation}
\label{eq:regret_IDS}
	\EE{R_T}\leq \sqrt{\pa{320K+21}T\log N}.
\end{equation}
\end{cor}
In particular, this recovers the minimax optimal rate of $ \OO(\sqrt{KT\log N}) $ for this problem.

\subsection{First-order bounds}\label{sec:firstorder}
We now present a more interesting result that replaces the dependence on $T$ in the previous bound by the 
cumulative loss of the best policy---constituting an instance-dependent guarantee that is often called 
\emph{first-order regret bound}. 
In particular, in the important class of ``noiseless'' problems where the optimal loss is zero, the result implies that 
\OIDS achieves constant regret.
\begin{thm}
Assume $ |\Theta| = N < \infty $, let $ L^{*} $ be such that $ \EE{\sum_{t=1}^{T}\lossmin_t(\thetastar)}\leq L^{*} $,
 and let $ \lambda= \sqrt{\frac{5\log N}{(500K + 108) L^{*}}} \land \frac{1}{250K + 54}$. Then the regret of \ROIDS with input parameter $ \mu = \frac{1}{10\lambda} $ and \VOIDS both satisfy
\label{thm:regret_IR_FOB}
\begin{equation}
\label{eq:regrt_IR_FOB}
\EE{R_T} \leq \sqrt{(2500K + 540)\log N L^{*}} + (1250K + 270) \log N.
\end{equation}
\end{thm}
We provide a proof sketch in Section~\ref{sec:FOB_proof_pt1}, with full details provided in 
Appendix~\ref{subsec:regret_IDS_FOB}.

\subsection{Infinite parameter spaces}
\label{subsec:metric_parameter_space}
We extend the result of Theorem~\ref{thm:regret_IR} to work for infinite parameter spaces. For simplicity, we focus on 
the case in which $ \Theta $ is a bounded subset of a finite-dimensional vector space. 
\begin{thm}
\label{thm:regret_IDS_metric}
Assume $ \Theta \subset \mathbb{R}^d $, $ max_{x,y \in \Theta} \norm{x-y} = 2R < \infty $. Assume that for all $ x\in 
X,a\in \A $, and $ L\in [0,1] $, the log-likelihood of the losses $ p(\cdot ,x,a,L) $ is $ C $-Lipschitz. Assume that a
ball of radius $ \frac{1}{CT} $ containing $ \thetastar $ is included in $ \Theta $ and set $ \lambda =  \sqrt{\frac{2d 
\log (RCT)}{(20K + \frac{21}{4})T}}$ and $ Q_1 $ a uniform prior on $ \Theta $. Then the regret of \ROIDS with input 
parameter $ \mu = \frac{1}{10 \lambda} $ and \VOIDS both satisfy
\begin{equation}
	\EE{R_T} \leq \sqrt{(160K + 42)dT \log (CRT)} + 1= \OO(\sqrt{dKT\log(CRT)}).
\end{equation}
\end{thm}
We provide a proof in Appendix~\ref{subsubsec:regret_IDS_metric}.

\subsection{Subgaussian losses}\label{subsec:subgaussian_losses}
We also extend the basic result of Theorem~\ref{thm:regret_IR} to work for a more general family of losses. In 
particular, we drop the assumption that the likelihood model is well-specified and allow the losses to be sub-Gaussian.
As the following result shows, we can still recover our regret bound of $\OO(\sqrt{KT\log N})$ with some minor tweaks 
of the algorithm and the analysis. The resulting method is called $\OIDSSG$, and is presented in 
Appendix~\ref{subsec:regret_IDS_SG} in full detail, along with the proof of the theorem below.
\begin{thm}
\label{thm:regret_IDS_SG}
Assume that the losses are $ v $-sub-Gaussian, that $ |\Theta| =N < \infty$ and set \linebreak $ \lambda =
\sqrt{\frac{\log N}{\left(\frac{1}{4} + 20 (v\land 1)(1+K)\right)T}}$. Then the regret of \ROIDSG with input parameter $ \mu = \frac{1}{80\lambda(v \land 1)} $ and \VOIDSG both satisfy
\begin{equation}
\label{eq:label}
\EE{R_T} \leq \sqrt{(1 + 80 (v\lor 1)(1+K))T\log N} = \OO(\sqrt{KT \log N}).
\end{equation}

\end{thm}

\section{Analysis}
This section provides an outline of the proofs of our main results. We first give a high-level overview of the key 
ideas that are shared in all proofs, and then fill in provide further technical details that are required to prove 
Theorems~\ref{thm:regret_IR} and~\ref{thm:regret_IR_FOB}. Theorems~\ref{thm:regret_IDS_metric} 
and~\ref{thm:regret_IDS_SG} are proved in Appendices~\ref{subsubsec:regret_IDS_metric} and~\ref{subsec:regret_IDS_SG}.

The core of our analysis is the following decomposition of the instantaneous regret in round $t$:
\begin{align}
 \EE{r_t(\pi_t; \thetastar)} &= \EE{\br_t(\pi_t)} + \EE{r_t(\pi_t;\thetastar) -\br_t(\pi_t)}\nonumber\\
    &= \EE{\br_t(\pi_t)} + \EE{\EEt{\lt(\thetastar,A_t)- \ltb(A_t)}} + \EE{\ltsb- 
\lts(\thetastar)}\label{eq:regret_decomposition}\\
    &= \EE{\ADEC_{\mu,t}(\pi_t) + \mu \SIG_t(\pi_t) + \UE_t + \OG_t}\nonumber.
\end{align}
Here, in the last line we have introduced the notations $\UE_t = \EEt{\lt(\thetastar,A_t)- \ltb(A_t)}$ to denote the 
\emph{underestimation error} of the losses incurred by our own policy $\pi_t$, and $\OG_t = \ltsb- \lts(\thetastar)$ as 
the \emph{optimatily gap} between the best loss possible in our mixture of models and the optimal loss attainable under 
the true parameter. The first term is small if the mixture model accurately evaluates the losses seen during learning 
(which is generally easy to ensure on average), and the second term is small if the model remains optimistic about the 
best attainable performance (which is facilitated by the optimistic adjustment to the posterior updates).
An important quantity in the analysis is the (true) \emph{information gain} of policy $\pi$ defined as
\begin{equation}
\label{eq:information_gain}
\IG_t(\pi) = \sum_{a\in\A} \pi(a) \int \DDH \left(p_t(\thetastar ,a,\cdot ),p_t(\theta,a,\cdot )\right)\,\dd 
Q_t^{+}(\theta).
\end{equation}
This quantity is closely related to the surrogate information gain that is optimized by our algorithm, and plays a key 
role in bounding the underestimation errors. In particular, the following simple lemma establishes a connection between 
the true and surrogate information gains:
\begin{lemma}
\label{lemma:Surrogate_information_gain_Hellinger}
	For any $t$ and policy $ \pi $, the information gain satisfies $\SIG_t(\pi) \leq 4 \IG_t(\pi)$.
\end{lemma}
The proof can be found in Appendix~\ref{subsubsec_Surrogate_information_gain_Hellinger}. Notably, 
the proof makes critical use of properties of the squared Hellinger 
distance, and is the main reason that the surrogate information gain is defined the way it is. In particular, the 
proof uses the fact that the Hellinger distance is a metric and as such it satisfies the 
triangle inequality---which is the reason that we were not able to go with the otherwise more natural choice of 
relative entropy in our definition of the information gain.

\subsection{The proof of Theorem~\ref{thm:regret_IR}}\label{sec:worst-case-analysis}
We first use the following worst-case bound on the underestimation error:
\begin{lemma}
\label{lemma:Underestimation_error_Hellinger}
For any $t$ and $\gamma >0$, the underestimation error is bounded as $
|\UE_t| \leq \frac{\gamma}{2} + \frac{\IG_t(\pi_t)}{\gamma}$.
\end{lemma}
The proof is relegated to Appendix~\ref{subsubsec:Underestimation_error_Hellinger}. Putting this bound together with the 
previous derivations, we get a regret bound that only depends on the \tADEC, the information gain and the optimality 
gap:
\begin{equation}
\label{eq:regret_decomposition_DEC_IG_OG}
	\EE{r_t} \leq \EE{\ADEC_{\mu,t}(\pi_t) + \left(4\mu + \frac{1}{\gamma}\right)\IG_t(\pi_t) + \OG_t} + 
\frac{\gamma}{2}.
\end{equation}
Following the terminology of \citet{Foste_G_Q_R_S23}, we will refer to the sum $\bpa{4\mu + 
\frac{1}{\gamma}}\IG_t(\pi_t) + \OG_t$ as the \emph{optimistic estimation error}. 
The following result establishes that the optimistic posterior 
updates can effectively control a quantity that is closely related to this term. 
\begin{lemma}
\label{lemma:EWF_bound_WC}
	Let $0<\eta < \frac{1}{2}$, $\lambda >0$, and $\beta=\frac{1}{1- 2 \eta}$.
	%and assume that the losses are well specified(\ref{ass:well_specified_losses}). 
	Then, 
the following inequality holds :
\begin{equation}
	\EE{\sum_{t=1}^{T} \pa{\frac{2\eta}{\lambda}\cdot\IG_t(\pi_t)  + \OG_t}} \leq \frac{\log 
\frac{1}{Q_1(\thetastar)}}{\lambda} + \frac{\lambda\beta T}{8}.
\end{equation}
\end{lemma}
See Appendix \ref{subsubsec:EWF_bound_WC} for the proof.
It remains to pick the hyperparameters in a way that the left-hand side matches the total 
optimistic estimation error, which is achieved when setting
way that $\frac{2\eta}{\lambda} = 4\mu + \frac{1}{\gamma}$. 
To make sure that this holds while minimizing the final constant, we choose $ \eta=\frac{1}{4}$, $\beta=2$, and $\gamma 
= \frac{1}{\mu}=10\lambda $. Plugging these constants into the bound above, and putting the result together with the 
bound of Equation~\eqref{eq:regret_decomposition_DEC_IG_OG} completes the proof of Theorem~\ref{thm:regret_IR}.

\subsection{The proof of Theorem~\ref{thm:regret_IR_FOB}}\label{sec:FOB_proof_pt1}
We start our analysis from the regret decomposition of Equation~\eqref{eq:regret_decomposition} and apply 
Lemma~\ref{lemma:Surrogate_information_gain_Hellinger} to obtain
\begin{equation*}
	\EE{r_t} \le \EE{\ADEC_{\mu,t}(\pi_t) + 4\mu \IG_t(\pi_t) + \UE_t + \OG_t}.
\end{equation*}
As before, we can control the ADEC of \OIDS by producing a suitable forerunner. In particular, we use the 
\emph{inverse-gap weighting} \IGW algorithm of \citet{Foste_K21}
\begin{lemma}
\label{lemma:ADEC_IGW_FOB}
The surrogate information ratio and averaged decision-to-estimation-coefficient of \VOIDS and \ROIDS satisfy for any $ \mu \geq 0 $
\begin{equation}
	4 \mu \ADEC_{\mu,t}(\ROIDS) \leq 4\mu \ADEC_{\mu,t}(\VOIDS) \leq \SIR_t(\VOIDS) \leq \SIR_t(\IGW) \leq 40K\min_{a\in \A}\ltb(a).  
\end{equation}

\end{lemma}
See Appendix~\ref{subsubsec:ADEC_IGW_FOB} for a definition of the (\IGW) algorithm and the proof.
The term on the right-hand side can be further bounded as 
\begin{align*}
	\ADEC_{\mu,t}(\pi_t) &\leq \frac{10K}{\mu} \min_{a}\ltb(a) \leq \frac{10K}{\mu}(\EEt{\ltb(A_t)})
	%\\
			     %&= \frac{10K}{\mu}(\EEt{\lt(\thetastar, A_t)-(\lt(\thetastar,A_t)- \ltb(A_t))}\\
			     = \frac{10K}{\mu}(\EEt{\lt(\thetastar,A_t)} - \UE_t)
\end{align*}

The final tool is a refined version of Lemma~\ref{lemma:Underestimation_error_Hellinger} that controls
the underestimation error in terms of the information gain and the current estimate of the loss.
\begin{lemma}
\label{lemma:Underestimation_error_FOB}
For any t and $ \gamma >0 $, the underestimation error is bounded as
\begin{equation}
	\UE_t \leq \frac{\IG_t(\pi_t)}{\gamma} + 2\gamma \EEt{\lt(\thetastar, A_t)}.
\end{equation}
\end{lemma}
See Appendix \ref{subsubsec:Underestimation_error_FOB} for the proof. Putting this together with the previous regret 
decomposition, as long as $ \frac{10K}{\mu} \leq 1 $, we get:
\begin{equation}
\label{eq:regret_decomposition_FOB_ana}
	\EE{r_t} \leq  \EE{  \left(4\mu + \frac{1}{\gamma}\cdot \left(1-\frac{10K}{\mu}\right)\right) \IG_t(\pi_t) + \OG_t + 
\left(2\gamma\left(1- \frac{10K}{\mu}\right) +\frac{10K}{\mu} \right)\lt(\thetastar ,A_t)},	
\end{equation}
As before, we will regard the term  $ \bpa{4 \mu + \frac{1}{\gamma}\cdot \bpa{ 1 - \frac{10K}{\mu} }}\IG_t 
+ \OG_t $ as the optimistic estimation error, and adapt Lemma~\ref{lemma:EWF_bound_WC} to provide a refined bound on 
this quantity:
\begin{lemma}
\label{lemma:EWF_bound_FOB}
	%Let $ \alpha = \frac{1}{2\eta}$. 
	Let $0<\eta < \frac{1}{2}$, $\lambda >0$, and $\beta=\frac{1}{1- 2 \eta}$. %
	Then, the optimistic estimation error satisfies
\begin{equation}
        \sum_{t=1}^{T} \pa{\frac{2\eta}{\lambda}\cdot\IG_t(\pi_t)  + \left( 1-\frac{\lambda \beta}{2} \right)\OG_t} \leq 
\frac{\log N}{\lambda} + \frac{\lambda\beta }{2}\sum_{t=1}^{T}\lts(\thetastar).
\end{equation}
\end{lemma}
See Appendix \ref{subsubsec:EWF_bound_FOB} for the proof.
The claim of the theorem is then proved by tuning the hyperparameters in a way that the quantity bounded in the previous 
Lemma matches the optimistic estimation error. We provide the details of this in Appendix~\ref{subsec:regret_IDS_FOB}.

\section{Discussion}\label{sec:discussion}
We have proposed a new analysis framework that bridges the concepts of information ratio and decision-estimation 
coefficient, and unifies the advantages of both frameworks. We provide some further discussion on our results below.

\paragraph{General bounded losses.} At the surface level, it may seem that our results only apply to well-specified 
models where the likelihood model correctly captures the distribution of the random losses. This is of course a very 
restrictive assumption. However, it is easy to see that our framework can tackle arbitrary bounded losses via a 
standard binarization trick~\citep{agarwal2013further}: supposing that the losses are bounded in $[0,1]$, they can be 
randomly rounded to $\ev{0,1}$ to apply \OIDS with a Bernoulli likelihood. It is easy to see that the regret bounds for 
these post-processed losses continue to hold for the original losses as well. We presume that our approach can be 
generalized beyond such sub-Bernoulli and sub-Gaussian losses to more general sub-exponential-family losses, but we 
leave the investigation of this generalization open for future work.

\paragraph{Multiplicative or additive tradeoff?} All of our results are stated in terms of both the surrogate 
information ratio, which measures the regret-to-information tradeoff multiplicatively, and the averaged DEC, which does so 
in an additive fashion. Based on these results, it is not immediately clear which of the two notions is more useful. 
Equations~\eqref{eq:IR_VIR} and~\eqref{eq:VIR_IR} suggest that the ADEC is always smaller than the information ratio, 
which may suggest that it may yield better guarantees. To a certain degree, \citet{Russo_V17} have already addressed 
this question: their Proposition~11 shows that measuring the regret-information tradeoff additively results in strictly 
\emph{worse} regret for a range of hyperparameter choices. While at the surface, this seems to defy the 
intuition provided our results, in reality their additive tradeoff is only vaguely related to the one we consider, and 
the regularization range for which the result holds does not seem to be practical in the first place. 
On the other hand, \citet{Foste_K_Q_R22} make a more robust argument against the information ratio in comparison with 
the DEC, showing that there are some hard problems for which the information ratio is infinite but 
the DEC remains finite (see their Section~9.3). Besides the fact that their information ratio is defined in an 
unorthodox way via the same conservative supremum as what appears in the definition of the DEC, this claim seems to miss 
some important follow-up work on \IDS that has already addressed this issue. Specifically, \citet{Latti_G21} have 
pointed out that the information ratio is only suitable for problems where the minimax regret is of the order $\sqrt{T}$ 
(which one can already notice by inspecting the general bound of Equation~\ref{eq:BR_bound_LIR}), and studying harder 
games with larger minimax regret may be done by introducing a generalized notion of information ratio that features a 
different power of the regret in the denominator. In the present paper, we decided to stay impartial and state our 
results for both flavors of optimistic \IDS, and we hope that this debate will progress productively in the future. 

\paragraph{Connection with the Bayesian DEC.} The attentive reader may have noticed that a notion closely related to 
our averaged DEC has already been mentioned in the original work of \citet{Foste_K_Q_R22}. Indeed, their Section~4.2 
proposes a Bayesian version of the \ETD algorithm that optimizes $\ADEC_{\gamma,t}(\cdot;Q_t)$, where $Q_t$ is the 
exact Bayesian posterior over the model parameters. They show that the resulting algorithm enjoys essentially 
the same guarantees on the Bayesian regret as the worst-case guarantees obtained by the standard \ETD method. Our 
approach effectively considers the same optimization objective, with the important change that the standard Bayesian 
posterior is replaced with the optimistic posterior of \citet{Zhang_21}. This not only strengthens the mentioned 
results of \citet{Foste_K_Q_R22} by removing the Bayesian assumption necessary for its analysis, but also allows us to 
obtain instance-dependent guarantees as well. We believe that the same instance-dependent improvements (and more) 
should be directly provable for the Bayesian \ETD method of \citet{Foste_K_Q_R22}, but we did not pursue this direction 
as we preferred to focus on pointwise regret guarantees this time. We provide further discussion on variants of 
the DEC (including the Bayesian DEC) and the information ratio in Appendix~\ref{app:complexities}.

\paragraph{Beyond contextual bandits.} For the sake of simplicity, we have presented our results within the relatively 
modest framework of contextual bandits. That said, it is clear that our framework can be generalized to the much 
broader setting of ``decision making with structured observations'' studied by \citet{Foste_K_Q_R22}, and that it 
can be used to prove regret bounds of the form of Theorem~\ref{thm:regret_IR} straightforwardly in said setting. 
However, so far we could only prove quantitative improvements over the DEC for contextual bandits, and thus we decided 
not to let down the reader by introducing a very general setting and then only providing interesting results in a 
narrow special case. Nevertheless, our results demonstrate that our framework can achieve strictly superior upper 
bounds on the regret in a highly nontrivial setting that has been studied extensively (see, e.g., 
\citealp{Agar_K_L_L17,Allen_B_L18,Foste_K21,Bubec_S22,Olkh_M_vE_N_W23}).

\paragraph{Lower bounds.} A very important question is if our notion of averaged DEC can also serve as a lower bound on 
the minimax regret like its original version proposed by \citet{Foste_K_Q_R22}. Since the ADEC is a lower 
bound on the DEC under a special choice of nominal model, we conjecture that it can also be used to lower bound the 
minimax regret in the same ``low-probability'' fashion as the original results of \citet{Foste_K_Q_R22}. On the same 
note, we remark that it seems unlikely that our DEC variant can be reconciled with the ``constrained DEC'' of 
\citet{Foste_G_H23}, which has so far yielded the tightest lower bounds on the regret within this family of complexity 
notions. Whether or not the averaging idea we advocate for in this paper will turn out to be useful for fully 
characterizing the minimax regret in sequential decision making remains to be seen.

\paragraph{Noiseless problems and Safe Bayes.} It is interesting to observe that the optimistic posterior updates used 
by our method simplify drastically in the special case of ``noiseless'' problems where $\ell^*(\theta,X_t) = 0$ holds 
for all $\theta$. This condition holds in two of the examples discussed in Appendix~\ref{app:examples}, and more 
broadly in all problems where the optimal policy is guaranteed to achieve zero loss under all candidate parameters 
$\theta$. As a more concrete example, we highlight the problem of bandit linear classification with surrogate losses, 
which satisfies this condition if the data is separable with a margin \citep{KST08,BOZ17,BPSzTWZ19}. In such noise-free 
problems, the optimistic posterior update collapses to $\frac{dQ_{t+1}^+}{dQ_t^+}(\theta) \propto 
(p_t(L_t|\theta,A_t))^\eta$, which is closer to the standard Bayesian update up to the important difference that it 
involves the ``stepsize'' parameter $\eta$. Interestingly, such ``generalized'' or ``safe'' Bayesian updates have been 
studied extensively in the context of statistical learning under misspecified models---see, e.g., 
\citet{Zhang_06,Zhang_06b,Gru12,dHKGM20}. This connection leads to a multitude of questions that we cannot hope to 
address in this short discussion, so we close with mentioning only one aspect that we find to be particularly exciting. 
Specifically, we wonder if the techniques established in these works could be useful for addressing misspecification in 
the context of sequential decision making under uncertainty, where this issue has been notoriously hard to formalize and 
handle \citep{Du_K_W_Y20,LSzW20,WASz21}. We leave this exciting question open for future research.

\acks{This project has received funding from the European Research Council (ERC) under the European Union’s Horizon 2020 
research and innovation programme (Grant agreement No.~950180). M.~Papini was supported by the European Union -- Next 
Generation EU within the project NRPP M4C2, Investment 1.,3 DD. 341 - 15 march 2022 -- FAIR -- Future Artificial 
Intelligence Research -- Spoke 4 -
PE00000013 - D53C22002380006.}

\bibliography{ref}
\clearpage
\appendix
\section{Examples}\label{app:examples}
The most appealing property of \IDS in the Bayesian setting is that it can take advantage of the structure of the 
problem at hand to achieve extremely good performance that is otherwise not achievable by methods like Thompson 
sampling or UCB. Indeed, unlike these methods, \IDS has the ability to pick actions that are not optimal under any 
statistically plausible model, but can reveal useful information about the problem. \citet{Russo_V17} demonstrate 
several examples of situations where \IDS provably achieves massive speedups via such queries. It is not clear that 
such speedups are achievable without Bayesian assumptions, although some evidence was offered by the work of 
\citet{Kirsc_K18} in the case of linear rewards. In this section, we demonstrate that our version of \IDS can fully 
reproduce the fast learning behavior of Bayesian \IDS on the original examples of \citet{Russo_V17}, thus suggesting 
that \OIDS may inherit many more good properties of its Bayesian counterpart than what our main theoretical results 
show. We also provide an additional example on which we demonstrate that \OIDS can outperform DEC-based methods by 
addressing the over-conservatism encoded in the definition of the DEC.

\subsection{Revealing action}
We first adapt the ``revealing actions'' example of the original work of \cite{Russo_V17}. This example features the 
action set $\mathcal{A} = \{0,1,\ldots,K\}$, the set of parameters $\Theta = \{ 1,\ldots,K \}$, and the 
loss function $ \ell(\theta,a) = \mathbb{I}_{\ev{a>0,\,a \neq \theta}} + \mathbb{I}_{\ev{a=0}}(1- 
\frac{1}{2^\theta})$. The losses are deterministic and the agent gets loss $0$ by picking the action corresponding to 
the unknown parameter $ \thetastar$. Action $0$ is special, it results in a large loss that however encodes the 
identity of the optimal action. Thus, the optimal exploration strategy is to pick this revealing action once, read out 
the identity of the optimal action, and play that action until the end of time. \citet{Russo_V17} show that \IDS 
follows this exact strategy, and here we show that \OIDS does the same when taking as input a (completely 
noninformative) uniform prior over the parameters.

To show this, we will compute for any action the surrogate reward and surrogate information gain under the optimistic 
posterior (which is identical to the uniform prior, given that we are in the first round). For $ a \neq 0 $, the 
surrogate regret is written as
\begin{align*}
	\br_1(a) &= \int_{\Theta} \ell(\theta,a) - \ell(\theta)\,dQ_0(\theta) = 
\frac{1}{K}\sum_{\theta=1}^{K}(1-\II{a=\theta})= 1- \frac{1}{K},
\end{align*}
while for the revealing action, the surrogate regret is 
\begin{equation*}
	\br_1(0) = 1 - \frac{1}{K}+ \frac{2^{-K}}{K}.
\end{equation*}
In particular $ \bar{r_t}(0) > \bar{r}_t(a)$ so the action $0$ has the worst expected reward under our model. As 
for the information gain, we an explicit computation of the Hellinger distance for $a\neq 0 $ shows
\begin{equation*}
	\IG_t(a) = \frac{1}{K}\cdot \left(1-\sqrt{\frac{1}{K}}\right) + \frac{K-1}{K}\cdot 
\left(1-\sqrt{\frac{K-1}{K}}\right) = \OO\left(\frac{1}{K}\right).
\end{equation*}
Meanwhile, for action $0$ we have
\begin{equation*}
	\IG_t(0) = 1 - \sqrt{\frac{1}{K}} = \Theta(1).
\end{equation*}

\subsection{Sparse linear model}
Our second example is a linear bandit problem where the action space corresponds to a finite subset of the 
Euclidean unit ball $ \mathcal{A} = \{ \frac{x}{\norm{x}_1}: x \in \{ 0,1 \}^d,\, x \neq 0 \} $, the parameter space 
consists of the set of coordinate vectors $ \Theta = \{ \theta^{\prime}\in \{ 0,1 \}^d, \norm{\theta^{\prime}}_1 =1 \} 
$, and the loss function is $\ell(\theta,a) = 1-\siprod{a}{\theta}$. As in the previous example, the losses are again 
deterministic. This is a linear bandit problem where the parameter $ \theta $ is known to be $1$-sparse. In particular, 
the optimal action under the model $ \theta $ consists in only selecting action $ a=\theta $ so any Thompson Samling 
based algorithm will only select one coordinate at a time and will take up to $ d $ steps to determine the true 
parameter $\thetastar$. In contrast, the optimal exploration policy will perform binary search on the action space and 
find the optimal action exponentially faster.

To investigate the behaviour of \OIDS on this problem, we will compute the surrogate regret and surrogate information 
gain of an action $ a $. Since our prior is uniform, we have
\begin{equation*}
	\br_1(a) = \bloss_1(a)= \PP{\siprod{\thetastar}{a}>0} \cdot \frac{1}{\norm{a}_1} = 
\frac{\norm{a}_1}{d}\cdot \frac{1}{\norm{a}_1} = \frac{1}{d}
\end{equation*}
and 
\begin{align*}
	\IG_1(a) &= \frac{\norm{a}_1}{d}\cdot \left( 1 - \sqrt{\frac{\norm{a}_1}{d}} \right) + \frac{d-\norm{a}_1}{d}\cdot 
\left( 1-\sqrt{\frac{d-\norm{a}_1}{d}} \right)	\\
		      &= 1 - \left( \frac{\norm{a}_1}{d} \right)^{\frac{3}{2}} - \left( 1- \frac{\norm{a}_1}{d} 
\right)^{\frac{3}{2}}
\end{align*}
Thus, the expected reward of all actions is the same, and the information gain is maximized for actions with norm 
$\onenorm{a} = \frac{d}{2}$. \IDS thus picks an action $A_1$ uniformly at random and updates the posterior as follows.
If the observed loss is $1$, all parameters with $\iprod{\theta}{A_1} > 0$ will be 
eliminated by the posterior update. If the observed loss is smaller than $1$, all parameters 
satisfying $\iprod{\theta}{A_1} = 0$ are excluded. The posterior is thus set as uniform over all surviving parameters 
and the process repeats. Continuing along the same lines, we can see that both versions of \OIDS will continue 
performing binary search and identify the true parameter in $ \log_2d $ time steps.

\subsection{Bandits with a revelatory zero}
Our final example is a multi-armed bandit problem where the losses keep looking exactly the same until a 
low-probability event happens that reveals the optimal action perfectly.
In this setup (vaguely inspired by Example~3.3 of \citealp{Foste_K_Q_R22}), $\Theta=[K]$, and the losses are defined as 
uniformly distributed random variables in $[0,1]$ for all 
actions except $a = \theta$. For this special action, the loss is defined as $B_t U_t $, with $U_t$ uniform on 
$[0,1]$, and $B_t$ is Bernoulli with mean $1-2\Delta \in [0,1]$. The mean loss for this action is $\frac 12 - 
\Delta$. For this 
model, there is essentially no way for any algorithm to discover the optimal action until the first time that a loss of 
zero is observed. In this case, the (optimistic) posterior immediately collapses on $\thetastar$. Consequently, \OIDS 
keeps drawing uniform random actions until the first zero is observed, and plays the optimal action in all remaining 
rounds. The number of time steps spent with uniform exploration are geometrically distributed with mean 
$\frac{K}{2\Delta}$, thus making for a total regret of approximately $\frac K2$. Note that in this instance, the 
optimistic adjustment to the posterior is not necessary as the optimal loss of all models are the same, so the 
performance of the algorithm is unaffected by the choice of $\lambda$ or $\mu$.

Interestingly, the \ETD algorithm of \citet{Foste_K_Q_R22} cannot take advantage of the structure of this problem 
so effectively. When using the posterior predictive distribution $\bp_t$ as the nominal model, the Hellinger 
distance will approximately behave as $\DDH\pa{p(\theta,a), \hp_t(a)} \approx \II{\theta \neq \thetastar}$ after 
observing the first zero. Thus, the worst-case DEC associated with policy $\pi$ is
written as
\begin{align*}
 \DEC_{\gamma}(\pi;\bp_t,\Theta) &= \sup_{\theta} \ev{\ell(\theta,\pi) - \ell(\theta,a_\theta) - \gamma \II{\theta \neq 
\thetastar}} = \sup_{\theta} \ev{\Delta \sum_{a\neq \theta} \pi(a) - \gamma \II{\theta \neq \thetastar}}
\\
&= \sup_{\theta} \ev{\Delta (1-\pi(\theta)) - \gamma \II{\theta \neq \thetastar}}.
\end{align*}
When $\gamma \ge \Delta$, the expression in the supremum can be positive for certain policies $\pi$ and parameters 
$\theta\neq\thetastar$, and thus the $\theta$ player will prefer picking $\theta \neq \thetastar$ for some policies. 
More precisely, the DEC for any policy will be given as
\[
 \DEC(\pi;\Theta,\hp_t) = \max\ev{\Delta (1-\min_{a\neq\thetastar} \pi(a)) - \gamma,\,  \Delta (1-\pi(\thetastar))}.
\]
In the extreme case $\gamma = 0$, the policy achieving maximum value is approximately uniform, and it approximates the 
optimal policy $\pi^*$ gradually as $\gamma$ increases. When $\gamma$ is large enough, the alternative $\theta \neq 
\thetastar$ stops being attractive to the max player and \ETD starts outputting $\pi^*$. This happens at the threshold 
$\gamma > \Delta$ at the latest. This observation matches the discussion of \citet[Example~3.3]{Foste_K_Q_R22} and 
\citet[p.~8]{Foste_G_H23}, who demonstrate the same threshold behavior of the DEC and point out that this leads to 
tight lower bounds, without discussing the potential shortcomings of \ETD that prevents it from obtaining tight upper 
bounds. It is easy to see that \ETD fails because of the over-conservative definition of the DEC: while there is 
sufficient evidence to reject all alternative parameters, \ETD still computes its optimization objective by taking a 
supremum over \emph{all} model parameters $\theta$, including ones that have already been ruled out by the 
observations. This clearly demonstrates the advantage of the surrogate model used by \OIDS, which computes its 
objective with the help of the optimistic posterior distribution that allows faster elimination of unlikely 
parameters.

\section{A tour of the complexity zoo of sequential decision-making}\label{app:complexities}
The DEC and \IDS frameworks that serve as foundations of our work (as described in Section~\ref{sec:frameworks}) come 
in a variety of flavors, with a number of possible choices for how one can trade off estimation errors and regret. We 
attempt to do justice to this literature here by discussing some of the most prominent variants of the DEC and the 
information ratio that have appeared in previous work.

\paragraph{Decoupling coefficient.} \citet{Zhang_21} defined the \emph{decoupling coefficient} as a 
technical tool for analyzing the regret of the so-called ``feel-good Thompson sampling'' (FGTS) algorithm that samples 
its actions $A_t$ by sampling $\theta_t \sim Q_{t}^+$ from the optimistic posterior (defined in the same terms as in 
our paper), and then following the policy $\pi_{\theta_t}$. The resulting policy is called $\pi_t$. We give a 
simplified definition here and refer the reader to Section~4 of \citet{Zhang_21} for the full details. In our notation, 
the decoupling coefficient $\delta_t^*$ is defined as the smallest constant $\delta$ such that the following inequality 
holds:
\[
 \EEs{\loss_t(\theta,\pi_\theta) - \loss_t(\theta_0,\pi_\theta)}{\theta \sim Q_t^+} \le \inf_{\mu \ge 0} \ev{\mu \sum_a 
\pi_t(a) \EEs{\bpa{\loss_t(\theta,a) - \loss_t(\theta_0,a)}^2}{\theta\sim Q_t^+} + \frac{\delta}{4\mu}}.
\]
Evaluating the infimum and reordering gives that the optimal value of $\delta_t^*$ is given as
\[
 \delta_t^* = \frac{\bpa{\EEs{\loss_t(\theta,\pi_\theta) - \loss_t(\theta_0,\pi_\theta)}{\theta \sim Q_t^+}}^2}{\sum_a 
\pi_t(a) \EEs{\bpa{\loss_t(\theta,a) - \loss_t(\theta_0,a)}^2}{\theta\sim Q_t^+}}.
\]
When specialized to the Bayesian setting where $\theta_0$ is sampled from the posterior distribution $Q_t$ and the 
optimistic posterior is replaced by the true posterior, this coefficient can be seen to be very closely related to the 
information ratio of \citet{Russo_V15,Russo_V17}. Indeed, if the losses are assumed to follow a Gaussian distribution, 
\citet{Neu_O_P_S22} point out that the decoupling coefficient becomes identical to the information ratio. Importantly, 
a limitation of the decoupling coefficient is that, without Bayesian assumptions, it cannot serve as a foundation for 
designing algorithms since it depends on the unknown parameter $\theta_0$ and as such it cannot be optimized in an 
\IDS-like algorithmic framework. This limits the usefulness of the decoupling coefficient to analyzing 
Thompson-sampling-like algorithms, as done by \citet{Zhang_21} and \citet{Agarw_Z22}.

\paragraph{Optimistic DEC.} \citet{Foste_G_Q_R_S24} define an \emph{optimistic} version of the original \DEC of 
\citet{Foste_K_Q_R22}, which makes two important changes to the standard definition. The first is that the optimistic 
\DEC (that we'll abbreviate as \ODEC below) replaces the reference model by a reference distribution $Q$ and also makes 
a change to the definition of regret in the tradeoff between low losses and low estimation errors. Specifically, the 
\ODEC associated with a reference distribution $Q$, a model class $\Theta$ and a policy $\pi$ is defined as
\[
  \ODEC_{\gamma,t}(\pi;\Theta,Q) = \sup_{\theta\in\Theta} \sum_{a} \pi(a) \EEs{\pa{\ell(\htheta,X_t,a) - 
\ell(\theta,X_t,\pi_\theta) - \gamma \DDH\pa{p_t(\theta,a), p_t(\htheta,a)}}}{\htheta\sim Q}.
\]
This quantity bears some superficial similarity with the ADEC in the sense that it replaces the reference model with a 
distribution over model that resembles the optimistic posterior our notion is based on. Also, the ODEC requires 
controlling the same notion of optimistic estimation error that our analysis makes use of 
(cf.~Section~\ref{sec:worst-case-analysis}). Importantly however, the main difference between the ODEC and the ADEC 
is that the former takes a supremum over model parameters as opposed to averaging them out according to the optimistic 
posterior. This supremum ultimately makes the \ODEC more similar to the \DEC than the ADEC, and inherits all the 
conservativeness of the \DEC.

\paragraph{Bayesian DEC.} Section 4.2 of \citet{Foste_K_Q_R22} defines a Bayesian version of the DEC, defined in terms 
of the posterior distribution $Q_t$ as
\[
  \BDEC_{\gamma,t}(\pi) = \sum_a \pi(a) \int \pa{\ell_t(\theta,a) - 
\lossmin_t(\theta) - \mu \DDH\bpa{\bp_{t}(a),p_t(\theta,a)}} \dd Q_t(\theta),
\]
where $\bloss_t(\pi) = \int \loss_t(\theta,\pi) \dd Q_t(\theta)$. This exactly matches our definition of the ADEC up to 
the important difference that the \BDEC is defined in terms of the true Bayesian posterior $Q_t$ and not the optimistic 
posterior $Q_t^+$. Theorem~4.2 of \citet{Foste_K_Q_R22} provides a Bayesian regret bound for this method that follows 
via a few lines of straightforward calculations, thanks to being free from the burden of having to control the 
optimistic estimation error as we need to do in our analysis. Indeed, in the Bayesian setting, the estimation errors 
are much easier to control via a straightforward application of the standard exponential weights analysis (as done in, 
e.g., \citet{Neu_O_P_S22}). Due to the counterexample provided by \citet{Zhang_21} for Thompson sampling, we find it 
unlikely that Bayesian algorithms derived from the BDEC can guarantee bounded worst-case regret in general.

%\paragraph{Constrained and average-constrained DEC.}

\paragraph{Algorithmic information ratio.} The work of \citet{XZ23} defined a notion of information ratio (called the 
``algorithmic information ratio'' or AIR) that does not require any Bayesian assumptions. Their methodology is 
closely related to the framework developed by \citet{Latti_G21} and \citet{Foste_R_S_S22} for \emph{adversarial} 
decision-making, and results in guarantees that hold without having to assume that the parameter $\theta_0$ remains 
fixed over time. Here we discuss a variant of the AIR that \citet{XZ23} call the ``model-index AIR'' (MAIR), and adapt 
the notations to allow a more evocative comparison between the MAIR and the ADEC. The MAIR for policy $\pi$ is defined 
in terms of a reference distribution $Q$ of models and a posterior belief distribution $\nu \in \Delta_{\Theta}$ as 
follows:
\[
  \MAIR_{\rho,\gamma}(\pi,\nu) = 
 \sum_a \pi(a) \int \pa{\ell_t(\theta,a) - 
\lossmin_t(\theta) - \mu \DDKL{\bp_{t}(a)}{p_t(\theta,a)}} \dd \nu(\theta) - \mu \DDKL{\nu}{\rho},
\]
with $\bp_t(a) = \int p_t(\theta,a) \dd \nu(\theta)$.
This quantity is very similar to the ADEC with the choice $\nu = Q_t^+$, up to the difference that the MAIR is defined 
in terms of the relative entropy instead of the Hellinger distance, and the presence of the additional regularization 
term $\mu \DDKL{\nu}{\rho}$ that penalizes beliefs that are very different from the reference distribution $\rho$ (and 
can be eliminated in principle by setting $\nu = \rho$). The key difference between the algorithms proposed by 
\citet{XZ23} and the ones studied in the present work and all the others discussed so far is in how a policy is derived 
from the complexity notions. Following \citet{Latti_G21} and 
\citet{Foste_R_S_S22}, \citet{XZ23} propose to choose their policies by finding a 
solution $(\pi_t,\nu_t)$ to the saddle-point optimization problem 
$\inf_{\pi\in\Delta_K} \sup_{\nu\in\Delta_{\Theta}} \MAIR_{\rho_t,\gamma}(\pi,\nu)$, where the reference distribution 
is updated as $\rho_{t+1} = \nu_t(\cdot|A_t,L_t)$ (the Bayesian posterior derived from the belief $\nu_t$ given the 
observations $L_t$ and $A_t$). This method has the merit of giving an update rule for the reference 
distribution that is defined without explicit Bayesian posterior updates, and being robust to adversarially chosen loss 
functions. On the other hand, the updates are arguably hard to compute in general (potentially even harder than 
computing exact posterior updates), and the robustness to adversarial data may come at the price of efficient 
exploitation of easy problem instances. Indeed, the supremum taken over parameter distributions $\nu$ may result in 
similarly conservative updates as other DEC-based algorithms, and may make it very challenging (if not impossible) to 
prove problem-dependent guarantees of the type we achieve in our work. We find it an intriguing open question to find 
out if this limitation of their method can be removed, and problem-dependent guarantees for algorithms based on the 
MAIR can be proved.

\allowdisplaybreaks

\section{Proofs of the main results}
We now give the complete proofs of our main results. We relegate most of the technical content into 
Appendix~\ref{app:technical_proofs} and only provide the main arguments here for better readability.

\subsection{The proof of Theorem~\ref{thm:regret_IR_FOB}}
\label{subsec:regret_IDS_FOB}
We continue from the regret bound obtained at the end of the analysis in Equation~
\ref{eq:regret_decomposition_FOB_ana}, that holds under the condition $ \frac{10K}{\mu} \leq 1 $:
\begin{align*}
	\EE{r_t} &\leq  \EE{  \left(4\mu + \frac{1}{\gamma}\cdot \left(1-\frac{10K}{\mu}\right)\right) \IG_t(\pi_t) + \OG_t 
+ 
\left(2\gamma\left(1- \frac{10K}{\mu}\right) +\frac{10K}{\mu} \right)\lt(\thetastar,A_t)}
\\
&\leq  \EE{  \left(4\mu + \frac{1}{\gamma} \right) \IG_t(\pi_t) + \OG_t 
+ 
\left(2\gamma +\frac{10K}{\mu} \right)\lt(\thetastar,A_t)},
\end{align*}
where in the last line we also used that $\IG_t$ and $\lt(\thetastar,A_t)$ are nonnegative to upper bound 
$1-\frac{10K}{\mu} \le 1$. In order to apply Lemma~\ref{lemma:EWF_bound_FOB}, we would like to manipulate the above 
expression so that the coefficients of $\IG_t$ and $\OG_t$ match. To this end, we use the condition that 
$\frac{\lambda \beta}{2} \leq \frac{1}{5}$, which ensures that $ 1 \leq \frac{1}{1-\frac{\lambda \beta}{2}}\leq 
\frac{5}{4}$ and thus we can continue the above bound as
\begin{equation*}
	\EE{r_t} \leq  \EE{\frac{5}{4}\cdot \left(   \left(4\mu + \frac{1}{\gamma}\right) \IG_t(\pi_t) + 
\left(1-\frac{\lambda \beta}{2}\right)\OG_t + \left(2\gamma +\frac{10K}{\mu} \right)\lt(\thetastar,A_t) \right)}.
\end{equation*}
To apply Lemma~\ref{lemma:EWF_bound_FOB}, we choose $ \eta = \frac{1}{4}, \beta=2, \gamma = \frac{1}{\mu} = 
10 \lambda $, and sum over all rounds to obtain
\begin{align*}
	\EE{R_T} &\leq \EE{\frac{5}{4}\cdot \frac{\log N}{\lambda} +\frac{5\lambda}{4}\sum_{t=1}^{T}\lts(\thetastar)+ (125K 
+ 25) \lambda \sum_{t=1}^{T}\lt(\thetastar,A_t) }\\
	&\leq \EE{\frac{5}{4}\cdot \frac{\log N}{\lambda} + (125K + 27) \lambda \sum_{t=1}^{T}\lt(\thetastar,A_t) 
},
\end{align*}
where we upper-bounded the optimal loss $\frac{5\lambda}{4} \ell^*_t(\thetastar)$ by $2\lambda \lt(\thetastar,A_t)$ in 
the last step.
Introducing the notation $ \wh{L}_T = \sum_{t=1}^{T}\lt(\theta*,A_t) $ and $L_t^{*} = \sum_{t=1}^{T} \lts(\thetastar)$, 
the two sides of the equation can be rewritten as
\[
 R_T = \wh{L}_T - L_t^{*} \le \EE{\frac{5}{4}\cdot \frac{\log N}{\lambda} + (125K + 27) \lambda 
\wh{L}_T }.
\]
Hence, after some reordering we arrive at
\begin{equation*}
	\EE{R_T}\cdot \left(1 - (125K + 27) \lambda\right) \leq  \EE{\frac{5}{4}\cdot \frac{\log N}{\lambda} + (125K + 27) 
\lambda L_T^{*}}.
\end{equation*}
If $ \lambda < \frac{1}{2(125K + 30)} $, we can divide both sides of the inequality by $ (1- 
(125K+27) \lambda) $ to obtain
\begin{equation*}
	\EE{R_T} \leq \EE{\frac{5}{2}\cdot \frac{\log N}{\lambda} + (250K + 54) \lambda L^{*}},
\end{equation*}
where $ L^{*} $ is an upper bound on $ \EE{L_T^{*}} $. Finally, we plug the value $ \lambda = \sqrt{\frac{5\log 
N}{(500K+108)L^{*}}}\land \frac{1}{250K + 54} $ to get the regret bound of Theorem \ref{thm:regret_IR_FOB}.

\subsection{The proof of Theorem~\ref{thm:regret_IDS_metric}}\label{subsubsec:regret_IDS_metric}
The only difference with the finite parameter space analysis is in the control of the optimistic estimation error. 
In particular, we only need to adapt our analysis of the optimistic posterior and Lemma~\ref{lemma:EWF_bound_WC} to get 
the regret bound claimed in Theorem~\ref{thm:regret_IDS_metric}. We do this with the following lemma.
\begin{lemma}
\label{lemma:EWF_bound_metric}
	Let $0<\eta < \frac{1}{2}$, $\lambda >0$, and $\beta=\frac{1}{1- 2 \eta}$, assume the hypothesis of Theorem~\ref{thm:regret_IDS_metric} hold.
	%and assume that the losses are well specified(\ref{ass:well_specified_losses}). 
	Then, 
the following inequality holds :
\begin{equation}
	\EE{\sum_{t=1}^{T} \pa{\frac{2\eta}{\lambda}\cdot\IG_t(\pi_t)  + \OG_t}} \leq \frac{d\log \frac{R}{\epsilon} }{\lambda} + \frac{\lambda\beta T}{8} + \left( \frac{\eta}{\lambda} +1 \right)\cdot CT \epsilon.
\end{equation}
\end{lemma}
The proof is found in Appendix~\ref{subsubsec:EWF_bound_metric}.
We can now put this together with the regret decomposition of Equation~\eqref{eq:regret_decomposition_DEC_IG_OG}. As in
the proof of Theorem~\ref{thm:regret_IR}, we need to pick the hyperparameters such that the optimistic estimation error
matches the left hand side of the previous lemma. The same choice of hyperparameters $ \eta=\frac{1}{4}, \beta=2, $ and
$ \gamma = \frac{1}{\mu} = 10\lambda $ combined with
Lemma~\ref{lemma:IR_TS} gives us the following bound
\begin{equation}
	\EE{R_T} \leq \lambda T(20K + \frac{1}{4} + 5) + \frac{d\log \frac{R}{\epsilon}}{\lambda} + \left( \frac{1}{4\lambda} + 1 \right)\cdot CT \epsilon.
\end{equation}
Picking $ \epsilon = 1/(CT) $ gives us 
\begin{equation}
	\EE{R_T} \leq \frac{2d \log RCT}{\lambda} + \lambda T \pa{20K + \frac{21}{4}} + 1,
\end{equation}
where we used $ \frac{1}{4} \leq d \log RCT $.
Finally picking $ \lambda = \sqrt{\frac{2d \log (RCT)}{T\pa{20K + \frac{21}{4}}}} $ recovers the claim of 
Theorem~\ref{thm:regret_IDS_metric}.

\subsection{The proof of Theorem~\ref{thm:regret_IDS_SG}}\label{subsec:regret_IDS_SG}
One of the appeals of our approach is that with minor tweaking, we can extend the previous guarantees so subgaussian
losses. To do that, we consider the following family of likelihoods:
\begin{equation*}
	p(c|\theta, x, a) \propto \exp{\left(-\frac{(c-\ell(\theta,x,a))^2}{2}\right)}.
\end{equation*}
We also readjust our definition of information gain for this setting by replacing the squared Hellinger distance by the 
square loss. In particular, the \emph{Gaussian surrogate information gain} is defined as 
\begin{align*}
	\SIGG_t(\pi) &= \sum_{a \in \mathcal{A}} \pi(a)\int \left(\lt(\theta,a)-\ltb(a)\right)^2\,\dd Q_t^{+}(\theta) 
\end{align*}
and the \emph{(true) Gaussian information gain} as 
\[
 	\IGG_t(\pi) = \sum_{a \in \mathcal{A}} \pi(a)\int \left(\lt(\theta,a)-\lt(\thetastar,a)\right)^2\,\dd 
Q_t^{+}(\theta)   .
\]
The surrogate information ratio and averaged DEC are adapted as
any policy $ \pi $
\begin{equation}
	\SIRG_t(\pi) = \frac{\br_t(\pi)}{\SIGG_t(\pi)} \qquad\mbox{and} \qquad\ADECG_{\mu,t}(\pi) = \br_t(\pi) - \mu\cdot 
\SIGG_t(\pi) .
\end{equation}
Then, we define the corresponding algorithm template (called Optimistic Information Directed Sampling for subgaussian 
losses, \OIDSSG) as the method that either picks $\pi_t$ as the minimizer of $\SIRG_t$ or $\ADECG_T$. The two varieties 
are referred to as \VOIDSG and \ROIDSG.

Replacing the surrogate information gain by its Gaussian counterpart, the regret decomposition of 
Equation~\eqref{eq:regret_decomposition} is still valid:
\begin{equation*}
	\EE{r_t}= \EE{\ADECG_t(\pi_t,\mu) + \mu \SIGG_t(\pi_t) + \UE_t + \OG_t}.
\end{equation*}
The surrogate and true information gains are related to each other by the following lemma:
\begin{lemma}
\label{lemma:Surrogate_information_gain_L2}
For any t and policy $\pi$, the information gain for Gaussians satisfies $ \SIGG_t(\pi) \leq 4 \IGG_t(\pi) $.
\end{lemma}
See Appendix~\ref{subsubsec:Surrogate_information_gain_L2} for the proof.
We also relate the underestimation error to the information gain through the following lemma
\begin{lemma}
\label{lemma:Underestimation_error_L2}
For any t and $ \gamma>0 $, the underestimation error is bounded as 
\begin{equation*}
	|\UE_t| \leq \frac{\gamma}{4} + \frac{\IGG_t(\pi_t)}{\gamma}.
\end{equation*}
\end{lemma}
The proof is presented in Appendix~\ref{subsubsec:Underestimation_error_L2}.
Putting these together, we get a regret bound that only depends on the average DEC, the information gain and optimality 
gap:
\begin{equation}
\label{eq:regret_decomposition_DEC_IG_OG_SG}
	\EE{r_t} \leq \EE{\ADECG_{\mu,t}(\pi_t) + \left( 4\mu + \frac{1}{\gamma}\right)\IGG_t(\pi_t) + \OG_t + 
\frac{\gamma}{4} }.
\end{equation}
We again refer to the sum $ \left( 4 \mu + \frac{1}{\gamma} \right)\IGG_t(\pi_t) $ as the optimistic estimation 
error and will control it through an analysis of the optimistic posterior adapted to the sub-Gaussianity of the losses. 
This is done in the following lemma, whose proof we relegate to Appendix~\ref{subsubsec:EWF_bound_SG}.
\begin{lemma}
\label{lemma:EWF_bound_SG}
Assume that the losses are $ v $ sub-Gaussian and that for all $ \theta \in \Theta, x \in X, a \in A $, $ 
\ell(\theta,x,a) \in [0,1] $, then setting $ \eta = \frac{1+\sqrt{1 - 1 \land v}}{2v} $ the following inequality holds :
\begin{equation}
\label{eq:EWF_bound_SG}
	\EE{\sum_{t=1}^{T}\frac{1}{16\lambda(v\lor 1)}\cdot \IGG_t(\pi_t) + \OG_t} \leq \frac{\log N}{\lambda} + 
\frac{\lambda T}{4}.
\end{equation}
\end{lemma}
Now we pick $ \mu = \frac{1}{\gamma} = \frac{1}{80 \lambda (v\lor 1)} $ and apply the previous lemma to obtain the bound
\begin{equation}
	\EE{R_T} \leq \EE{\sum_{t=1}^{T}\ADECG_{\frac{1}{80\lambda(v \lor 1)},t}(\pi_t)} + \frac{\log N}{\lambda} + \lambda T 
\left(\frac{1}{4} + 20 (v\lor 1)\right).
\end{equation}
It remains to bound the ADEC. We do this by exhibiting a ``forerunner'' algorithm that is able to control the 
\emph{Surrogate Information Ratio}. In particular, we use again the feel-good Thompson sampling (\FGTS) algorithm of 
\citet{Zhang_21} for this purpose.
\begin{lemma}
\label{lemma:IR_TS_SG}
	The surrogate information and averaged decision-to-estimation-coefficient of \OIDS and \VOIDS satisfy the 
following bound for any $ \mu > 0 $:
\begin{equation}
	4\mu \ADECG_{\mu,t}(\ROIDSG) \leq 4 \mu \ADECG_{\mu,t}(\VOIDSG) \leq \SIRG_t(\VOIDSG) \leq \SIRG_t(\textbf{FGTS}) = K 
\end{equation}
\end{lemma}
Putting everything together, we obtain the bound
\begin{equation}
	\EE{R_T} \leq \frac{\log N}{\lambda} + \lambda T \left( \frac{1}{4} + 20(v\lor 1)(1+K) \right),
\end{equation}
from which the bound claimed in Theorem~\ref{thm:regret_IDS_SG} follows by picking the optimal choice of $\lambda$.

\section{Technical proofs}\label{app:technical_proofs}
This section presents the more technical parts of the analysis, along with detailed proofs. The content is organized 
into four main parts: Appendix~\ref{app:UE} presents techniques for bounding the underestimation error, 
Appendix~\ref{app:SIG} provides techniques for relating the surrogate information gain to the true information gain, 
Appendix~\ref{subsec:EWF_Analysis} presents the analysis of the optimistic posterior updates to control the optimistic 
estimation error, and Appendix~\ref{app:SIR} provides bounds on the surrogate information ratio and the ADEC. All 
subsections include a variety of results, stated respectively for the worst-case bounds, first-order bounds, and 
subgaussian losses.

\subsection{Analysis of the Underestimation error}\label{app:UE}
\subsubsection{Worst case analysis: The proof of Lemma~\ref{lemma:Underestimation_error_Hellinger}}
\label{subsubsec:Underestimation_error_Hellinger}
We define the total variation distance between two distributions $P$, $Q$ sharing a common dominating measure $\lambda$ 
as
\begin{equation*}
	\TV(P,Q) = \frac{1}{2} \int |p(x)-q(x)|\,d \lambda(x), 
\end{equation*}
where $ p,q $ are their densities with respect to $ \lambda $.
The total variation distance can be upper bounded by the Hellinger distance as follows:
\begin{align*}
	\TV(P,Q) &= \frac{1}{2} \int \left|(\sqrt{p(x)} - \sqrt{q(x)})\cdot (\sqrt{p(x)} +\sqrt{q(x)})\right|\,d \lambda(x) 
\\
		&\leq \frac{1}{2} \sqrt{\int \left(\sqrt{p(x)}-\sqrt{q(x)}\right)^2\,d \lambda(x) \cdot \int \left(\sqrt{p(x)}+\sqrt{q(x)}\right)^2\,d \lambda(x)  } \\
		&\leq \frac{1}{2} \sqrt{2\DDH(P,Q)\cdot 2 \int (p(x) + q(x)) \,d \lambda(x) }\\
		&= \sqrt{2\DDH(P,Q)} \leq \frac{\gamma}{2} + \frac{\DDH(P,Q)}{\gamma}.
\end{align*}
Here, the first two inequalities follow from applying Cauchy--Schwarz, and the last one from the inequality of 
arithmetic and geometric means.
Thus, we proceed as
\begin{align*}
	|\UE_t| &= \left| \sum_{a} \pi_t(a) \int \lt(\thetastar,a) - \lt(\theta,a)\,\dd Q_t^{+}(\theta) \right| \\
	       &\leq \sum_{a} \pi_t(a) \int \bigl|\lt(\thetastar,a) - \lt(\theta,a)\bigr|\,\dd Q_t^{+}(\theta) \\
	       &= \sum_{a} \pi_t(a) \int \TV\bpa{\text{Ber}(\lt(\thetastar,a)),\text{Ber}(\lt(\theta,a))}\,\dd 
Q_t^{+}(\theta) \\
	       &\le \sum_{a} \pi_t(a) \int \TV\bpa{p_t(\thetastar,a),p_t(\theta,a)}\,\dd Q_t^{+}(\theta) \\
	       &\leq \frac{\gamma}{2} + \frac{\sum_{a} \pi_t(a)\int 
\DDH\bpa{p_t(\thetastar,a),p_t(\theta,a)}\,\dd Q_t^{+}(\theta)}{\gamma} \\
	       &= \frac{\gamma}{2} + \frac{IG_t}{\gamma}.
\end{align*}
The first inequality above uses the boundedness of the losses in $[0,1]$, the second inequality is the data-processing 
inequality for the total variation distance (applied on the noisy channel $X\ra Y$ that randomly rounds $X\in[0,1]$ to 
$Y\in\ev{0,1}$), and the last one is the inequality we have just proved above. This concludes the proof.

\subsubsection{Instance-dependent analysis: The proof of Lemma~\ref{lemma:Underestimation_error_FOB}}
\label{subsubsec:Underestimation_error_FOB}
This proof requires a more sophisticated technique based on careful specialized handling of the ``underestimated'' and 
''overestimated'' actions. The argument is vaguely inspired by the techniques of \citet{Bubec_S22} and 
\citet{Foste_K21}. Specifically, for a parameter $\theta$, we define $\A^-_\theta = 
\ev{a\in\A:\,\lt(\theta,a)< \lt(\thetastar,a)}$ as the set of actions where $\ell_t(\theta,a)$ underestimates 
$\ell_t(\thetastar,a)$. With this notation, we write
\begin{align*}
	\UE_t &= \sum_{a} \pi_t(a) (\lt(\thetastar,a)-\ltb(a))\\
	      &= \int \sum_{a} \pi_t(a) \bpa{\lt(\thetastar,a)-\lt(\theta,a))}\, \dd Q_t^{+}(\theta)\\
	      &\leq \int \sum_{a\in\A^-_\theta} \pi_t(a) \bpa{\lt(\thetastar,a)-\lt(\theta,a)}\, 
\dd Q_t^{+}(\theta)\\
	      &= \int \sum_{a\in\A^-_\theta} \pi_t(a)\cdot 
\frac{\sqrt{\gamma(\lt(\thetastar,a)+\lt(\theta,a))}}{\sqrt{\gamma(\lt(\thetastar,a)+\lt(\theta,a))}} 
(\lt(\thetastar,a)-\lt(\theta,a))\, \dd Q_t^{+}(\theta),
\end{align*}
where the inequality follows by dropping the negative terms of the sum.
Now, the inequality of arithmetic and geometric means implies that for any $ x,y \geq 0 $, $ xy \leq 
\frac{x^2 + y^2}{2} $. We apply it to $ x=2\sqrt{\gamma(\lt(\thetastar,a) + \lt(\theta,a))} $ and $ y= 
\frac{(\lt(\thetastar,a)-\lt(\theta,a))}{2\sqrt{\gamma(\lt(\thetastar,a)+ \lt(\theta,a))}} $ to obtain
\begin{equation*}
	\UE_t \leq \int \left( \gamma\sum_{a\in\A^-_\theta} \pi_t(a) \cdot 
\bpa{\lt(\thetastar,a) + \lt(\theta,a)} + \frac{1}{4\gamma}\sum_{a\in\A^-_\theta} 
\pi_t(a)\frac{\bpa{\lt(\thetastar,a)-\lt(\theta,a)}^2}{\lt(\thetastar,a) + \lt(\theta,a)}  \right)\, \dd 
Q_t^{+}(\theta).
\end{equation*}
To proceed, we use the inequality $ 
\frac{(\lt(\thetastar,a)-\lt(a))^2}{\lt(\thetastar,a) + \lt(\theta,a)} \leq 4 \DDH(p_t(\theta,a),p_t(\thetastar,a))$ 
that holds for all $a$ and $\theta$, and is proved separately as Lemma~\ref{lemma:DH_DT_relation}. Hence,
\begin{align*}
	\UE_t &\leq 2 \gamma \sum_{a} \pi_t(a)\lt(\thetastar,a) + \frac{1}{\gamma} \int \sum_{a} 
\DDH(p_t(\theta,a),p_t(\thetastar,a)) \, dQ_t^{+}(\theta)\\
	      &\leq 2 \gamma \sum_{a} \pi_t(a)\lt(\thetastar,a) + \frac{\IG_t}{\gamma},
\end{align*}
which concludes the proof.

\subsubsection{Subgaussian analysis: The proof of Lemma~\ref{lemma:Underestimation_error_L2}}
\label{subsubsec:Underestimation_error_L2}
	The claim follows from the following calculations:
\begin{align*}
	|\UE_t| &= \left|\sum_{a} \pi_t(a)\int \ell(\thetastar,a) - \ltb(a)\,dQ_t^{+}(\theta)\right| \\
		     &\leq \sum_{a} \pi_t(a)\int \left|\ell(\thetastar,a) - \ltb(a)\right|\,dQ_t^{+}(\theta) \\
		     &\leq \sqrt{\sum_{a} \pi_t(a)\int \left(\ell(\thetastar,a) - \ltb(a)\right)^2\,dQ_t^{+}(\theta)} \\
		     &= \sqrt{\IGG_t(\pi_t)} \\
		     &\leq \frac{\gamma}{4} + \frac{\IGG_t(\pi_t)}{\gamma}.
\end{align*}
Here, the second inequality is Cauchy--Schwarz and the last one is the inequality of arithmetic and geometric means.

\subsection{Analysis of the Surrogate Information Gain and the True Information Gain}\label{app:SIG}
\subsubsection{Bounded losses: The proof of 
Lemma~\ref{lemma:Surrogate_information_gain_Hellinger}}\label{subsubsec_Surrogate_information_gain_Hellinger}
\label{subsubsec:Surrogate_information_gain_Hellinger}
The claim is proved as
\begin{align*}
	\SIG_t(\pi) &= \sum_{a} \pi(a)\int \DDH\left(\ltb(a),\lt(\theta,a)\right) \,\dd Q_t^{+}(\theta) \\
			  &\leq 2\cdot \sum_{a}\pi(a)\int \DDH\left(\ltb(a),\lt(\thetastar,a)\right) \,\dd Q_t^{+}(\theta) 
			  \\
			  &\qquad\qquad+ 2\cdot 
\sum_{a}\pi(a)\int \DDH\left(\lt(\thetastar, a),\lt(\theta,a)\right) \,\dd Q_t^{+}(\theta) \\
			  &\leq  4\cdot \sum_{a}\pi(a) \int \DDH\bpa{\lt(\thetastar, a),\lt(\theta, 
a)}\,\dd Q_t^{+}(\theta) = 4\IG_t(\pi),
\end{align*}
where the first inequality critically uses that the Hellinger distance is a metric and as such it satisfies the 
triangle inequality, and thus $\DDH\pa{P,P'} \le 2 \DDH\pa{P,Q} + 2 \DDH\pa{Q,P'}$ holds for any $P$, $P'$ and $Q$ by 
an additional application of Cauchy--Schwarz. The final inequality then uses the convexity of the Hellinger distance 
and Jensen's inequality.

\subsubsection{Subgaussian losses: The proof of Lemma~\ref{lemma:Surrogate_information_gain_L2}}
\label{subsubsec:Surrogate_information_gain_L2}
The claims follows from writing
\begin{align*}
	\SIGG_t(\pi) &= \sum_{a} \pi(a) \int \left( \ltb(a)-\ell(\theta,a) \right)^2\,\dd Q_t^{+}(\theta) \\ 
			  &\leq 2\cdot \sum_{a} \pi(a) \int \left( \ltb(a)-\ell(\thetastar,a) \right)^2\,\dd Q_t^{+}(\theta) 
			  \\
			  &\qquad\qquad+ 
2\cdot \sum_{a} \pi(a) \int \left( \ell(\thetastar,a)-\ell(\theta,a) \right)^2\,\dd Q_t^{+}(\theta) \\
			  &\leq 4\cdot \sum_{a} \pi(a) \int \left( \ell(\thetastar,a)-\ell(\theta,a) \right)^2\,\dd 
Q_t^{+}(\theta) = 4 \IGG_t(\pi),
\end{align*}
where the first inequality comes an application of the triangle inequality and Cauchy--Schwarz, and the second one 
comes from the convexity of the squared loss and Jensen's inequality.

\subsection{Analysis of the Optimistic Posterior}% and proof of Lemma \ref{lemma:EWF_bound_DH}
\label{subsec:EWF_Analysis}
We start by providing a general statement about the properties of the optimistic posterior updates, which will then 
prove useful for bounding the optimistic estimation error.
\begin{lemma}
\label{lemma:EWF_bound}
Consider the optimistic posterior defined recursively by
\begin{equation}
	\frac{\dd Q_{t+1}^{+}}{\dd Q_t^{+}}(\theta) = \frac{\exp\left(-\eta \log(\frac{1}{p_t(L_t|\theta,A_t)}) - \lambda 
\lts(\theta)\right)}{\int\exp\left(-\eta \log(\frac{1}{p_t(L_t|\theta',A_t)}) - \lambda \lts(\theta')\right) \, 
\dd Q_t^{+}(\theta')},
\end{equation}
where $Q_1^{+} = Q_1$ is some prior distribution on $ \Theta $ and $ p_t(\cdot|\theta, a) \in \Delta_{\mathbb{R^{+}}} 
$ is the density the loss distribution associated with parameter $\theta$.
For any $ T>0 $, for any $ \alpha, \beta >0 $ such that $ \frac{1}{\alpha} + \frac{1}{\beta} =1 $, for any distribution 
$ Q^{*}\in \Delta (\Theta) $, and for any sequence of actions $ A_1,\ldots,A_T $ and losses $ L_1,\ldots,L_T $, the 
following inequality holds:
\begin{equation}
\begin{split}
\label{eq:EWF_bound}
	&- \frac{1}{\lambda\alpha} \sum_{t=1}^{T}\log\int p_t(\theta,A_t,L_t)^{\eta \alpha}\dd Q_t^{+}(\theta) - 
\frac{1}{\lambda\beta}\sum_{t=1}^{T} \log \int \exp{\left( -\lambda \beta \lts(\theta) \right)}\, \dd Q_t^{+}(\theta) 
\\
	&\qquad\qquad\leq \int \left( \frac{1}{\lambda \alpha}\cdot \sum_{t=1}^{T}\log\frac{1}{p_t(\theta,A_t,L_t)^{\eta 
\alpha}} + \sum_{t=1}^{T}\lts(\theta)  \right)\, \dd Q^{*}(\theta) + \frac{1}{\lambda}\cdot \DKL{Q^{*}}{Q_1} .
\end{split}
\end{equation}
\end{lemma}
\begin{proof}
We study the potential function $ \Phi $ defined for all $ c\in \mathbb{R}^{\Theta} $ as
\begin{equation*}
	\Phi(c) = \frac{1}{\lambda}\log \int_{\Theta} \exp(-\lambda c(\theta))\dd Q_1(\theta).
\end{equation*}
We define $ c_t(\theta) = \frac{\eta}{\lambda} \log \frac{1}{p_t(\theta,A_t,L_t)}  + \lts(\theta)$ and evaluate $ 
\Phi\left(\sum_{t=1}^{T}c_t\right) $:
\begin{equation*}
	\Phi\left(\sum_{t=1}^{T}c_t\right) = \frac{1}{\lambda}\log \int_{\Theta}\exp{\bigg(-\lambda 
\sum_{t=1}^{T}c_t(\theta)}\bigg)\dd Q_1(\theta) \geq - \int_{\Theta} \sum_{t=1}^{T}c_t(\theta)\dd Q^{*}(\theta) - 
\frac{\DKL{Q^{*}}{Q_1}}{\lambda}.
\end{equation*}
where the inequality is the Donsker-Varadhan variational formula~\citep[cf.~Section 4.9 in][]{Bouch_L_K13}. We also 
have 
\begin{align*}
	\Phi \left( \sum_{t=1}^{T}c_t \right) &= \sum_{t=1}^{T} \left( \Phi\left(\sum_{k=1}^{t}c_k\right) - \Phi \left( 
\sum_{k=1}^{t-1}c_k \right) \right) \\
					      &= \sum_{t=1}^{T} \frac{1}{\lambda} \log \frac{\int_{\Theta}\exp{\left( -\lambda 
\sum_{k=1}^{t}c_k(\theta) \right)}\dd Q_1(\theta)}{\int_{\Theta}\exp{\left( -\lambda \sum_{k=1}^{t-1}c_k(\theta) 
\right)}\dd Q_1(\theta)} \\
					      &= \sum_{t=1}^{T} \frac{1}{\lambda} \log \int_{\Theta} \frac{\exp{\left( -\lambda 
\sum_{k=1}^{t-1}c_k(\theta) \right)}}{\int_{\Theta} \exp{\left( -\lambda \sum_{k=1}^{t-1}c_k(\theta) 
\right)}\dd Q_1(\theta)}\cdot \exp{\left( -\lambda c_t(\theta) \right)} \dd Q_1(\theta) \\
					      &= \sum_{t=1}^{T}\frac{1}{\lambda} \log \int_{\Theta} \exp{\left( -\lambda c_t(\theta) 
\right)}\dd Q_t^{+}(\theta)\\
					      &= \sum_{t=1}^{T}\frac{1}{\lambda} \log \int_{\Theta} p_t(\theta,A_t,L_t)^{\eta} \cdot 
\exp{\left( -\lambda \lts(\theta) \right)}\dd Q_t^{+}(\theta),
\end{align*}
where the fourth equality is by definition of $Q_t^+$ and $c_t$.

We can now apply Hölder's inequality with $ \alpha, \beta>0 $ such that $ \frac{1}{\alpha} + \frac{1}{\beta} =1 $, 
obtaining
\begin{equation*}
	\Phi\left(\sum_{t=1}^{T}c_t\right) \leq \frac{1}{\lambda}\cdot \sum_{t=1}^{T}\left(  
\frac{1}{\alpha} \log \int_{\Theta}p_t(\theta,A_t,L_t)^{\eta \alpha}\dd Q_t^{+}(\theta) + \frac{1}{\beta} \log 
\int_{\Theta} \exp{\left( -\lambda \beta \lts(\theta) \right)}\dd Q_t^{+}(\theta) \right) .
\end{equation*}
Plugging both bounds together, we get the claim of the lemma.
\end{proof}

The following statement will be useful for turning the above guarantee into a bound on the information gain:
\allowdisplaybreaks
\begin{lemma}
\label{lemma:EWF_bound_IG}
For any $ t\geq 1 $ and any policy $ \pi_t \in \Delta (\A) $, the following inequality holds:
\begin{equation}
\label{eq:EWF_bound_IG}
	\EE{\IG_t(\pi_t)} \leq \EE{-\log \int_{\Theta}\left( \frac{p_t(\theta,A_t,L_t)}{p_t(\thetastar,A_t,L_t)} 
\right)^{\frac{1}{2}}\dd Q_t^{+}(\theta)}.
\end{equation}
\end{lemma}
\begin{proof}
Let $\tau$ be the dominating measure used to define the densities $p_t(\cdot|\theta,a)$. We write:
\begin{align*}
\EE{\IG_t(\pi_t)} &= \EE{\int_{\Theta}  \sum_{a} \pi_t(a)\DDH \left(p_t(\thetastar,a),p_t(\theta,a)\right)\,\dd 
Q_t^{+}(\theta)} \\
&= \EE{\int_{\Theta}  \sum_{a} \pi_t(a)\left( 1-\int_{\mathbb{R}} \left( p_t(c|\theta,A_t)p_t(c|\thetastar,A_t) 
\right)^{\frac{1}{2}}\, \dd \tau(c)\right)\,\dd Q_t^{+}(\theta)} \\
	&= \EE{\int_{\Theta}  \EEt{\int_{\mathbb{R}}\left( 1- \left( \frac{p_t(c|\theta,A_t)}{p_t(c|\thetastar,A_t)} 
\right)^{\frac{1}{2}}\right)\, p_t(c|\thetastar,A_t)\dd \tau(c)}\,\dd Q_t^{+}(\theta)} \\
&= \EE{\int_{\Theta}  \EEt{\int_{\mathbb{R}}\left( 1- \left( \frac{p_t(L_t|\theta,A_t)}{p_t(L_t|\thetastar,A_t)} 
\right)^{\frac{1}{2}}\right)\, p_t(L_t|\thetastar,A_t)}\,\dd Q_t^{+}(\theta)} \\
	&\le \EE{\EEt{-\log \int \left( \frac{p_t(L_t|\theta,A_t)}{p_t(L_t|\thetastar,A_t)} 
\right)^{\frac{1}{2}}\,\dd Q_t^{+}(\theta) }}\\
	&=\EE{-\log \int_\Theta \left( \frac{p_t(\theta,A_t,L_t)}{p_t(\thetastar,A_t,L_t)} 
\right)^{\frac{1}{2}}\,\dd Q_t^{+}(\theta) }.
\end{align*}
Here, we used the tower rule of expectation several times, and also the elementary inequality $\log(x) \leq x-1$ that 
holds for all $x$. This concludes the proof.
\end{proof}

\subsubsection{Worst case analysis: The proof of Lemma~\ref{lemma:EWF_bound_WC}}
\label{subsubsec:EWF_bound_WC}
\begin{lemma}
\label{lemma:Hoeffding_bound_OG}
	For any $t\geq 1$, $\beta,\lambda>0$, as long as $ \lts(\theta) \in [0,1] $ for all values of $ \theta $, the 
following inequality holds
\begin{equation}
	\frac{1}{\lambda\beta}\log \int_{\Theta} \exp{\left( -\lambda \beta \lts(\theta) \right)}\dd Q_t^{+}(\theta) \leq 
-\ltsb + \frac{\lambda \beta}{8}.
\end{equation}
\end{lemma}
\begin{proof}
	This is a direct consequence of Hoeffding's lemma for bounded random variables, see for example Section~2.3 of 
\citet{Bouch_L_K13}.
\end{proof}
The proof of Lemma~\ref{lemma:EWF_bound_WC} then follows directly by applying Lemma~\ref{lemma:EWF_bound} with $ \eta, 
\alpha $ such that $ \eta \alpha = \frac{1}{2} $ (which means $\beta=1/(1-2\eta)$) and with $ Q^{*} $ a dirac 
distribution in $ \thetastar $, and combining the result with  
Lemmas~\ref{lemma:EWF_bound_IG} and~\ref{lemma:Hoeffding_bound_OG} above.

\subsubsection{Instance dependent analysis and proof of Lemma \ref{lemma:EWF_bound_FOB}}
\label{subsubsec:EWF_bound_FOB}
\begin{lemma}
\label{lemma:FOB_bound_OG}
	For any $t\geq 1$, $\beta,\lambda>0$, as long as $ \lts(\theta) \in [0,1] $ for all values of $ \theta $, the 
following inequality holds
\begin{equation}
	\frac{1}{\lambda\beta}\log \int_{\Theta} \exp\bpa{\left( -\lambda \beta \lts(\theta) \right)}\dd Q_t^{+}(\theta) 
\leq -\ltsb\left(1 - \frac{\lambda \beta}{2}\right).
\end{equation}
\end{lemma}
\begin{proof}
	We use the two elementary inequalities $\log(x) \leq x - 1$ that holds for all $x\in\real$ and $e^{-x} \leq 
1-x+\frac{x^2}{2}$ that holds for all $x\ge 0$  to show
\begin{align*}
	\frac{1}{\lambda \beta} \log \int_{\Theta} \exp\bpa{-\lambda \beta \lts(\theta)}\dd Q_t^{+}(\theta) &\leq 
\frac{1}{\lambda \beta} \left( \int_{\Theta} 1 - \lambda \beta \lts(\theta) + \left(\frac{\lambda \beta}{2} 
\lts(\theta)\right)^2 \dd Q_t^{+}(\theta)-1 \right)\\
	&\leq \frac{1}{\lambda \beta} \left( \int_{\Theta} - \lambda \beta \lts(\theta) + \left(\frac{\lambda 
\beta}{2}\right)^2 \lts(\theta) \dd Q_t^{+}(\theta)\right)\\
	&= - \ltsb \left( 1 - \frac{\lambda \beta}{2} \right),
\end{align*}
where we used the fact that for all $ \theta \in \Theta$, we have $\lts(\theta) \in [0,1] $ and thus $ 
\lts(\theta)^2 \leq  \lts(\theta)$.
\end{proof}
We use again Lemma \ref{lemma:EWF_bound} with $ \eta, \alpha $ such that $ \eta \alpha =1/2$ and with $ Q^{*} $ a dirac 
distribution in $ \thetastar $. Then we apply Lemma~\ref{lemma:FOB_bound_OG} and Lemma~\ref{lemma:EWF_bound_IG} to 
conclude the proof of Lemma \ref{lemma:EWF_bound_FOB}.

\subsubsection{Subgaussian analysis: The proof of Lemma~\ref{lemma:EWF_bound_SG}}
\label{subsubsec:EWF_bound_SG}
\begin{lemma}
\label{lemma:EWF_IG_bound_SG}
Assume that the losses are $ v $ sub-Gaussian and that for all $ \theta\in \Theta, x \in \mathcal{X}, a \in \mathcal{A}, 
\ell(\theta,x,a) \in [0,1] $.
For any $ t \geq 1, \eta, \alpha \geq 0 $ such that $ \delta = \frac{\eta \alpha}{2} \left( 1 - \frac{\eta \alpha v}{2} 
\right) \geq 0 $ and any policy $ \pi_t \in \Delta(\mathcal{A}) $, the following inequality holds
\begin{equation}
\label{eq:EWF_IG_bound_SG}
	\delta(1-2\delta)\cdot \EE{\IGG_t(\pi_t)} \leq \EE{-\log \int_{\Theta} \left( 
\frac{p_t(\theta,A_t,L_t)}{p_t(\thetastar,A_t,L_t)} \right)^{\eta \alpha}\dd Q_t^{+}(\theta)}.
\end{equation}
\end{lemma}

\begin{proof}
	We remind the reader than $ \mathcal{F}_t = \theta(X_1,A_1,L_1,\ldots,X_{t-1},A_{t-1},L_{t-1})$ is the $ \sigma 
$-algebra generated by the interaction history between the learner and the environment up to the end of round t. By the 
tower rule of expectation, we have
\begin{align}
	&\phantom{{}={}}\EE{-\log \int_{\Theta} \left( \frac{p_t(\theta,A_t,L_t)}{p_t(\thetastar,A_t,L_t)} \right)^{\eta 
\alpha} \dd Q_t^{+}(\theta)}\nonumber\\
	&=\EE{\EE{-\log \int_{\Theta} \left( \frac{p_t(\theta,A_t,L_t)}{p_t(\thetastar,A_t,L_t)} \right)^{\eta \alpha} 
\dd Q_t^{+}(\theta)\middle\vert\mathcal{F}_{t-1},X_t,A_t}}\nonumber\\
	&\leq \EE{-\log \EE{\int_{\Theta} \left( \frac{p_t(\theta,A_t,L_t)}{p_t(\thetastar,A_t,L_t)} \right)^{\eta \alpha} 
\dd Q_t^{+}(\theta)\middle\vert\mathcal{F}_{t-1},X_t,A_t}}\nonumber\\
	&= \EE{-\log \int_{\Theta}\int_{\mathbb{R}} \left( \frac{p_t(\theta,A_t,L)}{p_t(\thetastar,A_t,L)} \right)^{\eta 
\alpha} d\mathcal{P}_{L_t|X_t, A_t}(L) \dd Q_t^{+}(\theta)}.\label{eq:Step_EWF_IG_SG}
\end{align}
Where the first inequality comes from Jensen's Inequality applied to the logarithm and $ \mathcal{P}_{L_t|X_t, A_t} $ is 
the conditional law of $ L_t $ given $ X_t $ and $ A_t $. We fix $ \theta \in \Theta $, drop the subscripts for 
simplicity and define $ \ell = \ell_t(A_t) $, $ \ell_0 = \ell_t(\thetastar,A_t)$ and $ \mathcal{P}_t =  
\mathcal{P}_{L_t|X_t, A_t}  $. Using the definition of the likelihood $ p_t $, we get 
\begin{align*}
	&\int \left(\frac{p_t(\theta,A_t,L)}{p_t(\thetastar,A_t,L)}\right)^{\eta \alpha}\,d\mathcal{P}_t(L) \\
	=& \int \exp{\left( -\eta \alpha \left( \frac{(L-\lt(\theta, A_t))^2}{2} + \frac{(L-\ell(\thetastar, 
A_t))^2}{2} \right) \right)}\,d\mathcal{P}_t(L)  \\
	=& \int \exp{\left( \frac{\eta \alpha}{2}(2L - \ell-\ell_0)\cdot (\ell-\ell_0) \right)}\,d\mathcal{P}_t(L) \\
								       =& \exp{\left(-\frac{\eta \alpha}{2}(\ell + \ell_0)\cdot 
(\ell-\ell_0)\right)}\cdot \int \exp{(\eta \alpha L(\ell - \ell_0))}\,d\mathcal{P}_t(L) \\
								       =& \exp{\left( \frac{\eta \alpha}{2}(\ell_0^{*2}-\ell^2) \right)}\cdot \int 
\exp{(\eta \alpha L(\ell - \ell_0))}\,d\mathcal{P}_t(L) \\
								       \leq & \exp{\left( \frac{\eta \alpha}{2}(\ell_0^{2}-\ell^2) \right)}\cdot 
\exp{\left( \eta \alpha \ell_0\cdot (\ell-\ell_0) \right)}\exp{\left( \frac{\eta^2 \alpha^2v}{2}(\ell-\ell_0)^{2} 
\right)} \\
								       \leq & \exp{\left( -(\ell-\ell_0)^2\cdot \frac{\eta \alpha}{2} \left(1 - 
\frac{\eta \alpha v}{2} \right)  \right)}.
\end{align*}
Now we define $ \delta = \frac{\eta \alpha}{2}(1 - \frac{\eta \alpha v}{2}) $ we have :
\begin{align*}
	&\int \left(\frac{p_t(\theta,A_t,L)}{p_t(\thetastar,A_t,L)}\right)^{\eta \alpha}\,d\mathcal{P}_t(L) \\
	\leq & \exp{\left( -(\ell-\ell_0)^2\cdot \delta \right)} \\
	\leq & 1 - \delta(\ell - \ell_0)^2 + \frac{\delta^2}{2} (\ell-\ell_0)^4 \\
	\leq & 1 - \delta(\ell - \ell_0)^2 + 4 \delta^2(\ell-\ell_0)^2 \\
	\leq & 1 - \delta(1-2 \delta)(\ell-\ell_0)^2.
\end{align*}
Where we use that $ |\ell - \ell_0|\leq 2 $.
Finally using that for any $ x>0$, $ \log x \leq x-1 $ and equation \ref {eq:Step_EWF_IG_SG}, we get the claim of the 
Lemma.
\end{proof}
It remains to pick the best values for $ \eta, \alpha $ and $ \beta $ and apply Lemma \ref{lemma:EWF_bound} with $ Q^{*} 
$ a dirac distribution in $ \thetastar $ and Lemma \ref{lemma:Hoeffding_bound_OG}.
To finish the proof of Lemma \ref{lemma:EWF_IG_bound_SG}, we combine the previous Lemma (\ref{lemma:EWF_IG_bound_SG}) 
with Lemma \ref{lemma:Hoeffding_bound_OG} and Lemma \ref{lemma:EWF_bound}. We want the quantity $ \delta(1-2\delta) $ to 
be as big as possible, this happens when $ \delta = \frac{1}{4} $. This is only possible if $ v\leq 1 $ and $ \frac{\eta 
\alpha}{2} = \frac{1 + \sqrt{1-v}}{2v}$. If $ v >1 $, our best choice of $ \frac{\eta \alpha}{2} $ is $ \frac{1}{2v} $ 
and in that case $ \delta(1-2\delta) = \frac{1}{4v}\left( 1 - \frac{1}{2v} \right) \geq \frac{1}{8v} $. Finally, uniting 
both cases, we set $ \alpha = \beta =2 $, $ \eta = \frac{1 + \sqrt{1- v\land 1}}{2v} $ and we have that $ 
\delta(1-2\delta) \geq \frac{1}{8(1\lor v)} $.

\subsubsection{Metric Parameter Analysis : the proof of Lemma~\ref{lemma:EWF_bound_metric}}
\label{subsubsec:EWF_bound_metric}
We start by a technical lemma on the Lipschtzness of the losses and the optimal losses.
\begin{lemma}
\label{lemma:Lipschitzness_losses}
	For any $ x,\theta,a $, $ \lt(\cdot,x,a) $ and $ \lts(\cdot ,x) $ are $ C $-Lipschitz.	
\end{lemma}
\begin{proof}
Let $ \tau $ be the measure against which the densities $ p(\cdot|\theta,x,a) $ are defined. Without loss of generality, 
we can assume that $ \int_{[0,1]}\dd \tau(c) =1 $.
Letting $ \theta_1, \theta_2 \in \Theta $, we have 
\begin{align*}
	|\ell(\theta_1,x,a)- \ell(\theta_2,x,a)| &= \left|\int_{[0,1]}c(p(c|\theta_1,x,a)-p(c|\theta_2,x,a)) \dd 
\tau(c)\right|\\
						 &\leq  \int_{[0,1]}\left|(p(c|\theta_1,x,a)-p(c|\theta_2,x,a))\right|\dd \tau(c)\\
						 &=  
\int_{[0,1]}\left|\exp{(\log(p(c|\theta_1,x,a)))}-\exp{(\log(p(c|\theta_2,x,a)))}\right|\dd \tau(c)\\
						 &\leq  \int_{[0,1]} C\norm{\theta_1 - \theta_2} \dd \tau(c)\\
						 &= C \norm{\theta_1-\theta_2},
\end{align*}
where the second inequality comes from the C-Lipschtzness of the composition of the exponential that is 1-Lipschitz on 
the negative numbers and the log likelihood that is C-Lipschitz. This proves the C-Lipschtzness of $ \lt(\cdot ,x,a) $. 
Now it easily follows that $ \lossmin(\cdot , x) $ is also C-Lipschitz, being an infimum of a family of C-Lipschitz 
functions.
\end{proof}
Now we introduce two further lemmas related to Lemma~\ref{lemma:EWF_bound} when $ Q^{*} $ is chosen as a uniform 
distribution on a ball of radius $ \epsilon $.
\begin{lemma}
\label{lemma:DKL_ball}
	Fix $ \thetastar \in \Theta$, and $ \epsilon >0 $, and assume that a ball including $ \thetastar $ with radius $ 
\epsilon $ is contained in $ \Theta $. Letting $ Q^{*} $ be the uniform distribution on such a  ball, we have
\begin{equation}
	\DKL{Q^{*}}{Q_1} = d\log\left(\frac{R}{\epsilon}\right).
\end{equation}
\end{lemma}
\begin{proof}
	Since both $ Q^{*} $ and $ Q_1 $ are uniform, the ratio of their density is equal to the ratio of the volume of $ 
\Theta $ and the volume of a ball of radius $ \epsilon $. Since $ \Theta $ is included in a ball of radius R, this ratio 
is bounded by $ (\frac{R}{\epsilon})^d $. Finally
\begin{equation*}
	\DKL{Q^{*}}{Q_1} = \int_{\Theta} \frac{\dd Q^{*}}{\dd Q_1}(\theta) 
\log\left(\frac{\dd Q^{*}}{\dd Q_1}(\theta)\right)\dd Q_1(\theta) \leq \log \left( \frac{R}{\epsilon} \right)^d 
\int_{\Theta} \dd Q^{*}(\theta) = d \log \left( \frac{R}{\epsilon} \right).
\end{equation*}
\end{proof}

\begin{lemma}
\label{lemma:EWF_bound_Lipschitz}
Under the same conditions as Lemma~\ref{lemma:DKL_ball}, we have
\begin{align}
	&\left|\int \left( \frac{1}{\lambda \alpha}\cdot 
\sum_{t=1}^{T}\log\frac{p_t(\thetastar,A_t,L_t)^{\eta\alpha}}{p_t(\theta,A_t,L_t)^{\eta 
\alpha}} + \sum_{t=1}^{T}\pa{\lts(\theta) - \lts(\thetastar)}  \right)\, \dd Q^{*}(\theta) \right|
	&\leq \left( \frac{\eta}{\lambda} + 1 \right)\cdot CT \epsilon.
\end{align}
\end{lemma}
\begin{proof}
This is a direct consequence of the Lipschitzness of the log-likelihood and Lemma~\ref{lemma:Lipschitzness_losses}.
\end{proof}
Putting Lemma~\ref{lemma:EWF_bound} together with this choice of $ Q^{*} $ and with Lemma~\ref{lemma:EWF_bound_IG} and Lemma~\ref{lemma:Hoeffding_bound_OG}, we finish the proof of Lemma~\ref{lemma:EWF_bound_metric}
\subsection{Upper bounds on the averaged DEC and the Surrogate Information ratio}\label{app:SIR}
Here we provide the technical tools to bound the surrogate information ratio and the averaged DEC for some 
appropriately chosen forerunner algorithms.

\subsubsection{Worst-case analysis: The proof of Lemmas~\ref{lemma:IR_TS} 
and~\ref{lemma:IR_TS_SG}}\label{subsubsec:IR_TS}
Here we study the performance of Thompson sampling as the forerunner algorithm, which will certify a bound on the 
surrogate information ratio of \OIDS. The Thompson sampling policy $ \pi_t $ works by sampling $ \theta_t $ 
according to the posterior $ Q_t^{+} $ and then playing the action $ A_t \in \argmin_{a} \lt(\theta_t, a) $. To 
facilitate the derivations below, we define $ a^*_t : \Theta\rightarrow \mathcal{A} $ 
the greedy action selector by $ a^*_t(\theta) = \argmin_a \lt(\theta,a)$ (with ties broken arbitrarily). 
By definition of the policy, sampling according to $ \pi_t $ is the same as sampling according to $ 
\dd Q_t^{+} $ and then applying the greedy action selector. More formally, for any measurable function $ f $, we have
\begin{equation*}
	\int_{\Theta} f(a^*_t(\theta))\dd Q_t^{+}(\theta) = \sum_{a} \pi_t(a)f(a).
\end{equation*}
Moreover, we have that $ \ltsb = \int_{\Theta} \lts(\theta)\dd Q_t^{+}(\theta) = \int_{\Theta} \lt(\theta, 
a^*_t(\theta))\dd Q_t^{+}(\theta)$. Putting these observations together, we can write the surrogate regret as
\begin{align}\label{eq:TS_regret}
	\br_t(\pi_t) &= \sum_{a} \pi_t(a)\bpa{\ltb(a)-\ltsb} = \int_{\Theta}\ltb(a^*_t(\theta)) - \lt(\theta, 
a^*_t(\theta)).
\end{align}
Observe that the regret is the difference of the expectation of the same function under the joint distribution of 
$\theta_t$ and $A_t$ and their product distribution, and thus measures the extent to which the two are ``coupled''. 
We will analyze this quantity by a decoupling argument inspired by \citet{Zhang_21} and \citet{Neu_O_P_S22}. 

For setting up the decoupling analysis, we first need some technical lemmas. We start by a corollary of the 
Fenchel--Young inequality for strongly convex functions that will come handy.
\begin{lemma}
\label{lemma:Fenchel_Young_divergence}
Let $I$ be an interval on the real line and let $\mathcal{D}: I^2 \rightarrow \mathbb{R} $ be a convex function 
satisfying the following conditions:
\begin{itemize}[itemsep=0pt,parsep=3pt,partopsep=6pt,topsep=6pt]
	\item For any $ y \in I $, the function $ x \rightarrow \mathcal{D}(x,y) $ is proper, closed and $C$-strongly 
convex.
	\item For any $ x \in I $, $ \mathcal{D}(x,x) = 0 $.
\end{itemize}
Then for any $ x, y \in I $ and any $ \mu \in \mathbb{R}  $ we have
\begin{equation}
\label{eq:Fenchel_Young_divergence}
(x-y)u \leq \mathcal{D}(x,y) + \frac{u^2}{2C}.
\end{equation}
\end{lemma}
\begin{proof}
	Let $ y \in I $. We compute the Legendre--Fenchel conjugate of $ x \rightarrow \mathcal{D}(x,y) $, defined for any 
$ u\in R $ as
\begin{equation*}
	\mathcal{D}^{*}(u, y) = \sup_{x\in I} \ev{xu - f(y)}.
\end{equation*}
Since $ y $ is a minimum of $ x\rightarrow \mathcal{D}(x,y) $ and $ \mathcal{D}(y,y) = 0 $, we have that $ 
\mathcal{D}^*(0,y) = 0 $. Moreover using Lemma 15 of \citet{Shwar07}, we directly have that $ \mathcal{D}^{*} $ is $ 
\frac{1}{C} $ smooth in its first coordinate and that $ \frac{\partial \mathcal{D}^{*}}{\partial u}(0,y) =y $, so that 
for any $ u\in \mathbb{R}$ we have
\begin{equation*}
	\mathcal{D}^{*}(u,y) \leq \mathcal{D}^{*}(0,y) + u \frac{\partial \mathcal{D}^{*}}{\partial u}(0,y) + \frac{u^2}{2C} 
\leq yu +\frac{u^2}{2C}.
\end{equation*}
Then, by the Fenchel--Young inequality, this implies the following for any $ x \in I $ and any $ u \in 
\mathbb{R}$:
\begin{equation*}
	x\cdot \mu \leq \mathcal{D}(x,y) + \mathcal{D}^{*}(u,y) \leq y\cdot u + \frac{u^2}{2C}.
\end{equation*}
This proves the statement.
\end{proof}

We use this inequality to prove the following general decoupling lemma that can handle arbitrary joint distributions of 
random variables.
\begin{lemma}
\label{lemma:General_Decoupling_Lemma}
Let $ \mathcal{D}:[0,1]^2 \rightarrow \mathbb{R} $ be $ C $-strongly convex and satisfy the same hypothesis as for the 
previous lemma. Let $ Q \in \Delta (\Theta) $, $ f: \Theta \times \mathcal{A} \rightarrow[0,1] $ and $ a^* : \Theta 
\rightarrow \mathcal{A} $. Assume $ f $ and $ a^* $ are measurable. We define $ \pi \in \Delta(\mathcal{A}) $ by $ 
\pi(a) = \int_{\Theta} \II{a^*(\theta)=a} \dd Q(\theta) $ and $ \bar{f}(a) = \int_{\Theta}f(\theta,a) \dd 
Q(\theta) $. Then for any $ \mu > 0 $ the following holds
\begin{equation}
\label{eq:General_Decoupling_Lemma}
\int_{\Theta}\bar{f}(a^*(\theta)) - f(\theta,a^*(\theta))\dd Q(\theta) \leq \mu \int_{\Theta} \sum_{a} \pi(a) 
\mathcal{D}(\bar{f}(a),f(\theta,a))\dd Q(\theta) + \frac{K}{2\mu C} 
\end{equation}
\end{lemma}
\begin{proof} 
We start by writing
\begin{align*}
	\int_{\Theta} \bar{f}(a^*(\theta)) - f(\theta,a^*(\theta)) &= \int_{\Theta} \sum_{a} \frac{\mu \pi(a)}{\mu \pi(a)} 
\II{a^*(\theta)=a} \left( \bar{f}(a) - f(\theta,a) \right) \dd Q(\theta) \\
							       &= \int_{\Theta} \sum_{a} \mu \pi(a)\left( \frac{\II{a^*(\theta)=a}}{\mu 
\pi(a)}  \left( \bar{f}(a) - f(\theta,a) \right) \right) \dd Q(\theta) \\
							       &\leq  \int_{\Theta} \sum_{a} \mu \pi(a)\left( \mathcal{D}(\bar{f}(a), f(\theta,a)) + 
\frac{\II{a^*(\theta)=a}}{2C\mu^2 \pi(a)^2} \right) \dd Q(\theta), \\
\end{align*}
where we used Lemma~\ref{lemma:Fenchel_Young_divergence} with $ u=\frac{\II{a^*(\theta)=a}}{\mu \pi(a)} $ in 
the last line. Finally, we have
\begin{align*}
	\int_{\Theta} \bar{f}(a^*(\theta)) - f(\theta,a^*(\theta)) &\leq \mu \int_{\Theta}\sum_{a} \pi(a) 
\mathcal{D}(\bar{f}(a), f(\theta,a))\dd Q(\theta) + \frac{1}{2 \mu C} \sum_{a} \int_{\Theta} 
\frac{\mathbb{I}_{a^*(\theta) =a}}{\pi(a)}\dd Q(\theta) \\
							       &\leq \mu \int_{\Theta}\sum_{a} \pi(a) \mathcal{D}(\bar{f}(a), f(\theta,a))\dd 
Q(\theta) + \frac{K}{2 \mu C},
\end{align*}
where we used $ \pi(a) = \int_{\Theta} \mathbb{I}_{a^*(\theta)=a}\dd Q(\theta) $ in the last line.
\end{proof}

To prove Lemma~\ref{lemma:IR_TS}, we use the above result with $ Q=Q_t^{+} $, $ f=\lt $ and $ a^* = a^^*_t $ and $\DD$ 
chosen as the squared Hellinger distance $\DDH$, which is $ \frac{1}{4} $-strongly convex in its first argument by 
Lemma~\ref{lemma:Strong_convexity_Hellinger} provided in Appendix~\ref{app:aux}. Thus, applying 
Lemma~\ref{lemma:General_Decoupling_Lemma} we get for any $ \mu >0 $ that
\begin{equation*}
	\br_t(\pi_t) \leq \mu \int_{\Theta}\sum_{a} \pi_t(a) \DDH(\ltb(a),\lt(\theta,a))\dd Q_t^{+}(\theta) + 
\frac{2K}{\mu}.
\end{equation*}
This concludes the proof of Lemma~\ref{lemma:IR_TS}. 

Lemma~\ref{lemma:IR_TS_SG} is proved by choosing $ \mathcal{D}(x,y) = (x-y)^2 $ that is $ 2 $-strongly convex in its 
first argument, which yields the advertised result as
\begin{equation*}
	\br_t(\pi_t) \leq \mu \int_{\Theta}\sum_{a} \pi_t(a) (\ltb(a)-\lt(\theta,a))^2\dd Q_t^{+}(\theta) + \frac{K}{4\mu}.
\end{equation*}

\subsubsection{Instance-dependent analysis: The proof of Lemma~\ref{lemma:ADEC_IGW_FOB}}
\label{subsubsec:ADEC_IGW_FOB}
This analysis uses the so-called \emph{inverse-gap weighting} algorithm of \citet{Abe_L99} as forerunner---see also the 
works of \citet{Foste_R20} and \citet{Foste_K21} that reignited interest in this method. Our analysis below is 
especially inspired by the latter work.

We define the inverse gap weighting policy with scale parameter $ \gamma $ and with respect to a nominal loss 
function $ f:\A 
\rightarrow \mathbb{R^{+}} $ as 
\begin{equation*}
	\pIGW_{\gamma,f}(a) = \begin{cases}
		\frac{f(b)}{Kf(b) + \gamma(f(b)-f(a))} &\text{if } a \neq b \\
1- \sum_{a \neq b} \pIGW_{\gamma,f}(a) &\text{if } a = b
\end{cases}
\end{equation*}
where $ b\in \argmin_a f(a) $ is fixed (with ties broken arbitrarily).
We fix $ \theta$
and apply Lemma~4 of \citet{Foste_K21} with nominal loss $ \ltb : A \rightarrow \mathbb{R} $ and true loss $ 
\lt(\theta): A \rightarrow \mathbb{R} $ to get
\begin{align*}
	\ltb(b) - \lt(\theta,a^{*}_t(\theta)) &\leq \frac{K}{4 \gamma} \ltb(b) + 2 \gamma\cdot \pIGW_{\gamma, 
\ltb}(a^{*}(\theta)) \frac{(\ltb(a^{*}(\theta))-\lt(\theta,a^{*}(\theta)))^2}{\ltb(a^{*}_t(\theta)) + 
\lt(\theta,a^{*}_t(\theta))}\\
					    &\leq \frac{K}{4 \gamma} \ltb(b) + 2 \gamma\cdot \sum_{a} \pIGW_{\gamma, \ltb}(a) 
\frac{(\ltb(a)-\lt(\theta,a))^2}{\ltb(a) + \lt(\theta,a)},
\end{align*}
where $ b \in \argmin_a \ltb(a) $ and $ a^{*}_t(\theta) \in \argmin_a\lt(\theta,a) $.
To proceed, we use that for any $ p,q \in [0,1] $, we have $ \frac{(p-q)^2}{p+q} \leq 4\cdot \DDH(\Ber(p),\Ber(q)) 
$ (cf.~Lemma~\ref{lemma:DH_DT_relation} in Appendix~\ref{app:aux}).
We combine this with the data processing inequality for $f$-divergences to obtain
\begin{equation}\label{eq:IGW_regret}
 \begin{split}
	\ltb(b) - \lt(\theta,a^{*}_t(\theta)))&\leq \frac{K}{4 \gamma} \ltb(b) + 8 \gamma\cdot \sum_{a} \pIGW_{\gamma, 
\ltb}(a) \DDH(\Ber(\ltb(a)),\Ber(\lt(\theta,a))) \\
	&\leq \frac{K}{4 \gamma} \ltb(b) + 8 \gamma\cdot \sum_{a} \pIGW_{\gamma, \ltb}(a) \DDH(\ptb(a,\cdot 
),p_t(\theta,a,\cdot )).
 \end{split}
\end{equation}

On the other hand, we can rewrite the surrogate regret of the inverse gap weighting policy as
\begin{align*}
	\br_t(\pIGW_{\gamma, \ltb}) &= \int \sum_{a} \pIGW_{\gamma,\ltb}(a) (\ltb(a)-\lts(\theta))\,\dd Q_t^{+}(\theta)\\
		   &= \int \sum_{a} \pIGW_{\gamma,\ltb}(a) (\ltb(a)-\lt(\theta,a^{*}_t(\theta)))\,\dd Q_t^{+}(\theta)\\
		   &= \int \sum_{a \neq b} \pIGW_{\gamma,\ltb}(a) (\ltb(a)-\ltb(b))\,\dd Q_t^{+}(\theta) +\int \sum_{a} \pi(a) 
(\ltb(b)-\lt(\theta,a^{*}_t(\theta)))\,\dd Q_t^{+}(\theta).
\end{align*}
The second term in the above decomposition can be bounded using Equation~\eqref{eq:IGW_regret}. As for the first term, 
we can exploit the definition of the policy to write
\begin{equation*}
	\sum_{a \neq b} \pIGW_{\gamma,\ltb}(a) (\ltb(a)-\ltb(b)) = \sum_{a \neq b} \frac{\ltb(b)(\ltb(a)-\ltb(b))}{K\ltb(b) 
+ \gamma(\ltb(a)-\ltb(b))} \leq \frac{K \ltb(b)}{\gamma}.
\end{equation*}
Putting these bounds together gives
\begin{align*}
	\br_t(\pIGW_{\gamma,\ltb}) &\leq \frac{K\ltb(b)}{\gamma} + \frac{K\ltb(b)}{4\gamma} + 8\gamma\cdot \int 
\sum_{a}\DDH(\ptb(a,\cdot ),p_t(\theta,a,\cdot ))\,\dd Q_t^{+}(\theta)\\
				   &\leq \frac{5K\ltb(b)}{4\gamma} + 8\gamma\cdot \SIG_t,
\end{align*}
Optimizing over $ \gamma $, we get the claim of Lemma~\ref{lemma:ADEC_IGW_FOB}.

\subsection{Auxiliary results}\label{app:aux}
\begin{lemma}[Proposition 3 ~\citealp{Foste_K21}]
\label{lemma:DH_DT_relation}
	For any $ p,q \in [0,1]$, we have
\begin{equation*}
	\frac{(p-q)^2}{p+q} \leq 4 \DDH(\Ber(p),\Ber(q)).
\end{equation*}
\end{lemma}
\begin{proof}
The statement follows from the simple calculation
\begin{equation*}
\DDH(p,q) \geq \frac{1}{2}(\sqrt{p}-\sqrt{q})^2 = \frac{1}{2} \left( 
\frac{(\sqrt{p}-\sqrt{q})(\sqrt{p}+\sqrt{q})}{\sqrt{p}+\sqrt{q}} \right)^2 = \frac{1}{2} 
\frac{(p-q)^2}{(\sqrt{p}+\sqrt{q})^2} \geq \frac{1}{4} \frac{(p-q)^2}{p+q},
\end{equation*}
where the last step uses the elementary inequality $ (x+y)^2 \leq 2(x^2 + y^2) $ that holds for any $ x,y $.
\end{proof}
\begin{lemma}
\label{lemma:Strong_convexity_Hellinger}
For any fixed $ q\in [0,1] $, the function $p \mapsto \DDH(\Ber(p),\Ber(q))$ is $ \frac{1}{4} $-strongly convex.
\end{lemma}
\begin{proof}
The proof is based on showing that the second derivative of the function of interest is uniformly lower bounded by a 
positive constant. This follows from calculating the first derivative as 
\begin{align*}
	\frac{\partial \DDH(p,q)}{\partial p} &= \frac{1}{2}\left( -\sqrt{\frac{q}{p}} + \sqrt{\frac{1-q}{1-p}} \right),
\end{align*}
and then lower-bounding the second derivative as
\begin{align*}
	\frac{\partial ^2\DDH(p,q)}{\partial ^2p} &= \frac{1}{4} \left( \sqrt{\frac{q}{p^3}} + \sqrt{\frac{1-q}{(1-p)^3}} 
\right) \geq \frac{1}{4} \left( \sqrt{q}+\sqrt{1-q} \right) \geq \frac{1}{4}.
\end{align*}
This inequality is tight when $ q=0 $ or $ q=1 $.
\end{proof}

\end{document}